\documentclass{article}

\usepackage[preprint]{neurips_2026}

\usepackage[utf8]{inputenc} 
\usepackage[T1]{fontenc}    
\usepackage{hyperref}       
\usepackage{url}            
\usepackage{booktabs}       
\usepackage{amsfonts}       
\usepackage{nicefrac}       
\usepackage{microtype}      
\usepackage{xcolor}         
\usepackage{amsmath,amssymb,amsthm}
\usepackage{xspace}

\usepackage{multirow}

\usepackage{algorithm}
\usepackage{algpseudocode}

\newcommand{\best}[1]{\textbf{#1}}
\newcommand{\second}[1]{\underline{#1}}

\usepackage{graphicx}
\usepackage{tikz}
\usetikzlibrary{arrows.meta,fit,calc,backgrounds}

\usepackage{pifont}

\newcommand{\xmark}{\ding{55}}
\let\checkmark\cmark

\hypersetup{
  colorlinks=true,
  linkcolor=blue,
  citecolor=blue,
  urlcolor=blue,
}

\newcommand{\eps}{\varepsilon}
\newcommand{\Cost}{\mathrm{Cost}}

\newcommand{\mbe}{\textsc{MBE}}

\newcommand{\msma}{\textsc{MSMA}}

\newcommand{\cfs}{\textsc{CFS}\xspace}
\newcommand{\cfsm}{\text{\textsc{CFS}}}

\newcommand{\cfsvar}[2]{\ensuremath{\cfsm\,(#1\times #2)}}
\newcommand{\cfsvartag}[3]{\ensuremath{\cfsm\,(#1\times #2_{\mathrm{#3}})}}

\newcommand{\cfssubvartag}[4]{\ensuremath{\cfsm_{\mathrm{#1}}\,(#2\times #3_{\mathrm{#4}})}}

\newcommand{\AvgAUROC}{\mathrm{AvgAUROC}}
\newcommand{\AvgWorstAUROC}{\mathrm{AvgWorstAUROC}}

\newcommand{\diffpath}{\textsc{DiffPath}}
\newcommand{\ddpmood}{\textsc{DDPM-OOD}}
\newcommand{\lmd}{\textsc{LMD}}

\newcommand{\gepc}{\textsc{GEPC}}
\newcommand{\scoped}{\textsc{SCOPED}}
\newcommand{\eigenscore}{\textsc{EigenScore}}

\newtheorem{theorem}{Theorem}

\newtheorem{remark}{Remark}

\newtheorem{proposition}{Proposition}

\title{Backbone-Equated Diffusion OOD via Sparse Internal Snapshots}

\author{%
Yadang Alexis Rouzoumka$^{1,2}$ \\
DEMR, ONERA \& SONDRA \\
Université Paris-Saclay \\
\texttt{yadang-alexis.rouzoumka@centralesupelec.fr, rouzoumkaalexis@yahoo.fr}
\And
Jean Pinsolle$^{2}$ \\
SONDRA, CentraleSupélec \\
Université Paris-Saclay
\And
Eugénie Terreaux$^{1}$ \\
DEMR, ONERA \\
Université Paris-Saclay
\AND
Christèle Morisseau$^{1}$ \\
DEMR, ONERA \\
Université Paris-Saclay
\And
Jean-Philippe Ovarlez$^{1,2}$ \\
DEMR, ONERA \& SONDRA \\
Université Paris-Saclay
\And
Chengfang Ren$^{2}$ \\
SONDRA, CentraleSupélec \\
Université Paris-Saclay
}

\begin{document}

\maketitle

\begin{abstract}
Fair comparison between diffusion-based OOD detectors is challenging, as conclusions can vary with backbone choice, corruption parameterization, and test-time budget. We address this issue through a \emph{Mutualized Backbone-Equated} (\mbe) protocol that aligns canonical corruption levels and logical test-time cost across diffusion backbones. Within this setting, we introduce \emph{Canonical Feature Snapshots} (\cfs), a family of detectors that probes a frozen diffusion backbone using only a tiny number of native internal activations at canonical low-noise levels. On a controlled CIFAR-scale benchmark, the strongest one-forward CFS variant is \(\cfs(1{\times}2)\), while an even smaller decoder-only variant remains highly competitive. This shows that much of the relative-OOD signal exposed by frozen diffusion backbones is concentrated in a small number of sparse internal states, rather than requiring full denoising trajectories or high-capacity downstream heads. We further provide a local diagnostic theory explaining these observations through conditional encoder-decoder complementarity, diagonal-score separation, and low-noise corruption stability. The official implementation is available at
\url{https://github.com/RouzAY/cfs-diffusion-ood/}.
\end{abstract}

\section{Introduction}
\label{sec:intro}

Out-of-distribution (OOD) detection asks whether a model can identify inputs that fall outside the regime where its predictions should be trusted. In diffusion models, once the backbone is frozen, the same image can be probed across a structured family of corruption levels, producing a hierarchy of hidden states before the final denoising output.

Most diffusion-based OOD methods probe this backbone in \emph{output space}: through reconstructions, denoising residuals, trajectories, output-side geometry, or posterior-consistency scores \citep{mahmood2021msma, graham2023ddpmodd, liu2023lmd, gao2023diffguard, heng2024diffpath, rouzoumka2026gepc, barkley2025scoped, shoushtari2025eigenscore}. This has produced useful detectors, but it leaves open a more basic question: Within a frozen diffusion backbone, where does the most useful information for OOD detection reside?

We argue that this question is obscured by two coupled issues. The first is \emph{protocol confounding}: in diffusion OOD, conclusions can depend strongly on checkpoint family (e.g., DDPM vs.\ EDM), corruption coordinates, and test-time budget, much as protocol choices already matter in post-hoc OOD benchmarking more broadly \citep{yang2022openood, zhang2024openoodv15}. The second is representation mismatch: output-space summaries rely on compressed readouts and may miss discriminative information still present in internal states, consistent with the growing view of diffusion models as representation learners \citep{yang2023repfusion, luo2023diffhyper, yu2025repa}.


We study this question under a controlled evaluation setting. First, we introduce a \emph{Mutualized Backbone-Equated} (MBE) protocol that aligns checkpoint-family policy, canonical corruption levels, and logical test-time cost across diffusion backbones. Second, within this setting, we propose \emph{Canonical Feature Snapshots} (CFS), a family of OOD detectors that probes a frozen diffusion backbone through a tiny number of aligned internal activations.

On a controlled CIFAR-scale benchmark, the strongest one-forward operating point is \(\cfs(1{\times}2)\), while a decoder-only variant,\(\cfs_{\mathrm{dec}}(1{\times}1)\), remains highly competitive. Under controlled evaluation, useful OOD signal in frozen diffusion backbones is already strongly concentrated in a tiny number of sparse native internal snapshots.

To explain these trends, we develop a measurable local-testing view of sparse diffusion probing. It yields three diagnostic principles: conditional encoder-decoder complementarity, diagonal-score separation, and low-noise corruption stability. The theory is local rather than universal, but it produces directly estimable quantities for selecting hooks, levels, and encoder-decoder pairings, and we show empirically that these diagnostics track downstream OOD behavior across both improved-diffusion and EDM backbones.





\paragraph{Contributions.}
\begin{itemize}
    \item \textbf{A controlled protocol and sparse detector.}
    We formulate diffusion OOD comparison under a \emph{Mutualized Backbone-Equated} (MBE) protocol, and introduce \emph{Canonical Feature Snapshots} (CFS), a detector family that probes frozen diffusion backbones through a tiny number of aligned internal activations.

    \item \textbf{Evidence that the OOD signal is highly concentrated.}
    Under MBE, \(\cfs(1{\times}2)\) is a strong one-forward operating point, while the single late-decoder probe \(\cfs_{\mathrm{dec}}(1{\times}1)\) remains highly competitive, showing that a relevant signal is strongly concentrated in sparse native snapshots.

    \item \textbf{A local diagnostic explanation.}
    We develop a measurable local-testing view that explains conditional encoder-decoder complementarity, diagonal-score separation, and low-noise stability, and show that its diagnostics track empirical behavior.
\end{itemize}

\section{Related Work and Positioning}
\label{sec:related}

\noindent\textbf{Post-hoc OOD detection and protocol sensitivity.}
A large literature studies post-hoc OOD scores for pretrained discriminative models, including Mahalanobis distances, energy scores, activation shaping, virtual-logit matching, and nearest-neighbor geometry \citep{lee2018mahalanobis, liu2020energy, sun2021react, wang2022vim, sun2022knn, lu2025oodsurvey}. Benchmarking efforts such as OpenOOD and OpenOOD v1.5 show how strongly protocol choices can affect conclusions \citep{yang2022openood, zhang2024openoodv15}. These concerns are even sharper in diffusion OOD, where checkpoint family, corruption coordinates, and test-time budget introduce additional confounders.


\noindent\textbf{Diffusion OOD in output space.}
Most diffusion-based OOD methods probe denoisers through output-space quantities or output-derived summaries. \msma\ aggregates multiscale score norms \citep{mahmood2021msma}; \ddpmood\ and \lmd\ use denoising, reconstruction, or inpainting behavior \citep{graham2023ddpmodd, liu2023lmd}; likelihood-based diffusion OOD has also been studied \citep{goodier2023likelihood}; DiffGuard adds conditional guidance \citep{gao2023diffguard}; and \diffpath\ summarizes denoising trajectories from a single unconditional diffusion backbone \citep{heng2024diffpath}. Other methods use output-side geometric or consistency structure: \gepc\ measures transformation-induced posterior-consistency violations in denoiser outputs \citep{rouzoumka2026gepc}, while SCOPED and EigenScore study output-side geometry or uncertainty \citep{barkley2025scoped, shoushtari2025eigenscore}. Despite their differences, these methods share the same basic viewpoint: the OOD signal is extracted primarily from denoiser outputs, reconstructions, trajectories, consistency residuals, or output-side geometry.

\noindent\textbf{Diffusion models as representation learners.}
A parallel literature views diffusion models as representation learners. Pretrained diffusion backbones provide useful internal features for downstream tasks \citep{yang2023repfusion}, and semantically meaningful descriptors can be consolidated from multi-layer and multi-timestep states \citep{luo2023diffhyper}. Related work also suggests encoder/decoder asymmetry: Faster Diffusion reports that encoder activations vary less across denoising steps than decoder activations \citep{li2024faster}. While not an OOD result, this supports the idea that different internal states play different functional roles. REPA further argues that strong hidden representations are important for generation quality \citep{yu2025repa}. Recent OOD work also supports representation-space modeling beyond raw-pixel likelihoods \citep{ding2025repr_ood, jarve2025latent_density_ood}. These works motivate our representation-first viewpoint.

\noindent\textbf{Positioning.}
We study diffusion OOD under a \emph{Mutualized Backbone-Equated} protocol and ask: under a shared-source, backbone-equated, budget-accounted comparison, how much relative-OOD signal is already present in a tiny number of sparse native frozen snapshots? CFS answers this question with a deliberately minimal detector family: no reconstruction module, no guidance machinery, no reverse-path recursion, and no high-capacity downstream head. 
Table~\ref{tab:positioning} summarizes \cfs positioning.

\begin{table}[t]
\centering
\caption{\textbf{Positioning relative to diffusion-based OOD detection.} Our distinction is not merely that internal features can help, but that under a shared-source, backbone-equated protocol, a tiny number of sparse native frozen snapshots already captures a strong relative-OOD signal.}
\label{tab:positioning}
\small
\setlength{\tabcolsep}{4.2pt}
\resizebox{\textwidth}{!}{%
\begin{tabular}{lcccc}
\toprule
Method & Primary probe & Native internal features & Reconstruction / guidance & Shared-source protocol \\
\midrule
\msma\ \citep{mahmood2021msma} & multiscale output descriptor & \xmark & \xmark & \xmark \\
\ddpmood\ \citep{graham2023ddpmodd} & reconstruction / manifold & \xmark & \checkmark & \xmark \\
\lmd\ \citep{liu2023lmd} & inpainting / reconstruction & \xmark & \checkmark & \xmark \\
DiffGuard \citep{gao2023diffguard} & guided output probing & \xmark & \checkmark & \xmark \\
\diffpath\ \citep{heng2024diffpath} & denoising trajectory & \xmark & \xmark & \xmark \\
\gepc\ \citep{rouzoumka2026gepc} & output consistency & \xmark & \xmark & \xmark \\
\scoped\ \citep{barkley2025scoped} & score geometry & \xmark & \xmark & \xmark \\
\eigenscore\ \citep{shoushtari2025eigenscore} & covariance geometry & \xmark & \xmark & \xmark \\
DLSR \citep{yang2024dlsr} & feature reconstruction & \checkmark & \checkmark & \xmark \\
\midrule
\textbf{CFS (ours)} & sparse native snapshots & \checkmark & \xmark & \checkmark \\
\bottomrule
\end{tabular}}
\end{table}

\section{Mutualized Backbone-Equated Protocol (\mbe)}
\label{sec:mbe}

OOD benchmarking has shown that evaluation details can dominate perceived progress \citep{yang2022openood, zhang2024openoodv15}. In diffusion OOD, this problem is amplified by additional degrees of freedom, including checkpoint family, corruption parameterization, and hidden test-time budget. We therefore evaluate all methods under a shared protocol designed to remove these confounders.

\subsection{Canonical corruption and cross-backbone alignment}
\label{sec:mbe_canonical}

A first ingredient of \mbe\ is a shared corruption view across backbone families. Diffusion backbones evaluate corrupted versions of an input across a family of noise levels \citep{ho2020ddpm, nichol2021improved, karras2022edm}, but their native interfaces differ: some expose discrete timesteps \(t\), while others use continuous noise scales such as \(\sigma\). To compare OOD detectors across backbones, we therefore work with a common canonical corruption parameterization.

For a clean sample $\mathbf{x}_0$, we write corruption as
\begin{equation}
\mathbf{x}_\lambda = a(\lambda)\mathbf{x}_0 + b(\lambda)\boldsymbol{\eps}, \qquad \boldsymbol{\eps} \sim \mathcal N(0,I),
\label{eq:canon_corruption}
\end{equation}
and use the canonical coordinate, with a slight abuse of notation,
\begin{equation}
\lambda := \log \frac{a(\lambda)^2}{b(\lambda)^2},
\label{eq:logsnr_def}
\end{equation}
i.e., logSNR. Large \(\lambda\) corresponds to cleaner observations and small \(\lambda\) to noisier ones.

For improved-diffusion checkpoints, \(a_t=\sqrt{\bar\alpha_t}\), \(b_t=\sqrt{1-\bar\alpha_t}\), and \(\lambda_t=\log \frac{\bar\alpha_t}{1-\bar\alpha_t}\), so a desired canonical level is matched to the nearest native timestep in logSNR space. EDM-style backbones instead expose a continuous noise input; our adapter maps each canonical \(\lambda\) to the appropriate backbone-specific model input and returns coefficients satisfying Eq.~\eqref{eq:canon_corruption}. Canonicalization matters because many diffusion OOD methods are intrinsically multilevel: without a shared corruption coordinate, methods may be compared at levels that are not actually matched in corruption strength.

\subsection{Controlled comparison under \mbe}
\label{sec:mbe_controlled}

Under \mbe, a method is evaluated under the same source-family policy, preprocessing, canonical corruption levels, ID/OOD splits, and logical test-time cost as its competitors. The purpose of \mbe\ is scientific comparison, not per-backbone peak tuning. A diffusion OOD method can otherwise appear stronger for reasons unrelated to its scoring rule: a better-matched checkpoint, a different input normalization, a different corruption coordinate, or a larger hidden test-time budget.

Concretely, \mbe\ routes all methods through the same canonical corruption semantics and shared adapter interface, and evaluates them under the same split policy and logical budget accounting. It does not force identical native implementations. Detailed canonicalization, adapter outputs, and baseline implementation taxonomy are deferred to Appendix~\ref{app:protocols} and Appendix~\ref{app:taxonomy_baselines}.


For discrete backbones, mapping a continuous logSNR grid to native timesteps can produce duplicates. We therefore distinguish the candidate grid resolution \(K_{\mathrm{grid}}\) from the number of effective canonical levels \(K_c\) actually used by a method; Appendix~\ref{app:kgrid}--\ref{app:canon_mapping} gives the precise construction.

We report the logical test-time cost as
\[
\Cost_m = \#F_m + \#J_m,
\]
where \(\#F_m\) is the number of backbone forward evaluations per image and \(\#J_m\) the number of Jacobian-type evaluations when applicable. In the main comparisons of this paper, the dominant variation is in \(\#F\), and all compared methods have \(\#J=0\).

\section{Method: Canonical Feature Snapshots (\cfs)}
\label{sec:method}

\subsection{Canonical Feature Snapshots}

We ask whether a tiny number of native frozen internal activations already captures useful OOD signal once protocol confounders are removed. Let \(P_\star\) denote the source distribution used to train a frozen diffusion checkpoint, and let \(P\) and \(Q\) denote the evaluation ID and OOD datasets. We do not treat the checkpoint as an  OOD oracle for \(P_\star\); instead, we use it as a frozen representation map and define OOD relative to the evaluation reference bank \(P\).


A \cfs\ instance is specified by a small set of canonical levels \(\Lambda\), a small set of native internal hooks \(\mathcal H\), and an ID-only scoring head on the resulting pooled slot descriptors. For each \(\lambda\in\Lambda\), we form
\[
\mathbf x_\lambda
=
a(\lambda)\mathbf x_0
+
b(\lambda)\boldsymbol{\varepsilon},
\qquad
\boldsymbol{\varepsilon}\sim\mathcal N(\mathbf 0,\mathbf I),
\]
run one forward pass through the frozen diffusion backbone, and pool each selected activation
\(\mathbf u_{\lambda,h}(\mathbf x_\lambda)\) into a descriptor
\[
\mathbf z_{\lambda,h}(\mathbf x_0)
:=
\operatorname{Pool}
\!\left(
\mathbf u_{\lambda,h}(\mathbf x_\lambda)
\right).
\]
The sparse representation is
\[
\boldsymbol{\Phi}(\mathbf x_0)
=
\big[
\mathbf z_{\lambda,h}(\mathbf x_0)
\big]_{(\lambda,h)\in\mathcal S},
\qquad
\mathcal S\subseteq \Lambda\times\mathcal H.
\]
OOD detection then reduces to fitting ID-only slot statistics on \(\boldsymbol{\Phi}(P)\) and scoring deviation from that reference geometry.

The paper focuses on two operating points: \(\cfs(1{\times}2)\), which uses one deep encoder and one late decoder hook at a single low-noise level, and \(\cfs_{\mathrm{dec}}(1{\times}1)\), a decoder-only variant. Because all retained hooks at a fixed level are captured within the same backbone forward pass, logical cost depends only on the number of selected levels:
\[
\#F_{\cfs}=|\Lambda|,\qquad \#J_{\cfs}=0.
\]

Figure~\ref{fig:cfs_overview} summarizes this pipeline.

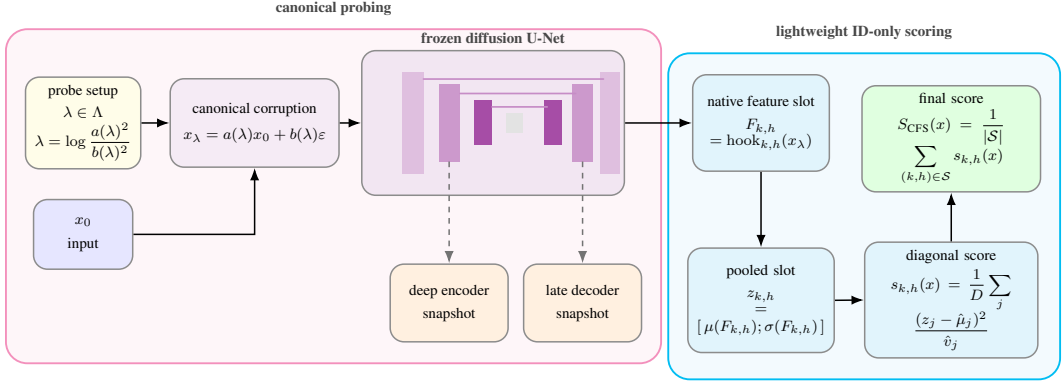
\begin{figure}[t]
\centering
\resizebox{\linewidth}{!}{%
\begin{tikzpicture}[
    >=Latex,
    font=\footnotesize,
    box/.style={
        draw=black!45,
        rounded corners=2.6mm,
        thick,
        align=center,
        inner sep=4pt
    },
    infobox/.style={
        box, fill=yellow!10,
        minimum width=2.15cm,
        minimum height=1.65cm
    },
    inputbox/.style={
        box, fill=blue!10,
        minimum width=1.85cm,
        minimum height=1.30cm
    },
    procbox/.style={
        box, fill=violet!8,
        minimum width=3.15cm,
        minimum height=1.65cm
    },
    unetframe/.style={
        box, fill=violet!10,
        minimum width=4.90cm,
        minimum height=2.70cm
    },
    hookbox/.style={
        box, fill=orange!12,
        minimum width=2.20cm,
        minimum height=1.45cm,
        anchor=south
    },
    headbox/.style={
        box, fill=cyan!10,
        minimum width=2.30cm,
        minimum height=1.95cm,
        anchor=south
    },
    topheadbox/.style={
        box, fill=cyan!10,
        minimum width=2.55cm,
        minimum height=1.80cm
    },
    outbox/.style={
        box, fill=green!12,
        minimum width=2.75cm,
        minimum height=1.95cm,
        anchor=south
    },
    arr/.style={-Latex, thick},
    darr/.style={-Latex, thick, dashed},
    taparr/.style={-Latex, thick, dashed, rounded corners=7pt}
]

\node[infobox] (meta) at (1.35,1.30)
{probe setup\\[2pt]
$\lambda\in\Lambda$\\[2pt]
$\displaystyle \lambda=\log\frac{a(\lambda)^2}{b(\lambda)^2}$};

\node[procbox] (corr) at (4.55,1.30)
{canonical corruption\\[4pt]
$\displaystyle x_\lambda=a(\lambda)x_0+b(\lambda)\varepsilon$};

\node[unetframe] (unet) at (9,1.30) {};

\node[topheadbox] (slotfeat) at (14,1.30)
{native feature slot\\[3pt]
$\displaystyle F_{k,h}$\\[-1pt]
$\displaystyle=\mathrm{hook}_{k,h}(x_\lambda)$};

\node[inputbox] (x0) at (1.35, -0.78) 
{$x_0$\\[3pt]input};

\node[hookbox] (enc) at (8.18,-2.80)
{deep encoder\\[3pt]snapshot};

\node[hookbox] (dec) at (10.67,-2.80)
{late decoder\\[3pt]snapshot};

\node[headbox, text width=2.45cm] (pool) at (14,-3)
{pooled slot\\[2pt]
$\displaystyle z_{k,h}$\\[-1pt]
$\displaystyle =[\,\mu(F_{k,h});\sigma(F_{k,h})\,]$};

\node[headbox, text width=2.95cm] (head) at (17.55,-3.10)
{diagonal score\\[2pt]
$\displaystyle s_{k,h}(x)=\frac{1}{D}\sum_j$\\[-1pt]
$\displaystyle \frac{(z_j-\hat\mu_j)^2}{\hat v_j}$};

\node[outbox, text width=3.05cm] (ood) at (17.55, 0)
{final score\\[2pt]
$\displaystyle S_{\cfs}(x)=\frac{1}{|\mathcal S|}$\\[-1pt]
$\displaystyle \sum_{(k,h)\in\mathcal S} s_{k,h}(x)$};

\draw[arr]  (meta.east) -- (corr.west);
\draw[arr]  (x0.east)  -| (corr.south);
\draw[arr]  (corr.east) -- (unet.west);
\draw[arr]  (unet.east) -- (slotfeat.west);




\draw[arr]  (slotfeat.south) -- (pool.north);
\draw[arr]  (pool.east) -- (head.west);
\draw[arr]  (head.north) -- (ood.south);

\node[font=\bfseries\footnotesize, text=black!72]
    at ($(unet.north)+(0,0.22)$) {frozen diffusion U-Net};

\fill[violet!25] ($(unet.center)+(-1.70, 0.95)$) rectangle ($(unet.center)+(-1.30,-0.95)$);
\fill[violet!40] ($(unet.center)+(-1.00, 0.72)$) rectangle ($(unet.center)+(-0.62,-0.72)$);
\fill[violet!70] ($(unet.center)+(-0.35, 0.42)$) rectangle ($(unet.center)+(-0.02,-0.42)$);

\fill[gray!22]   ($(unet.center)+(0.25, 0.18)$) rectangle ($(unet.center)+(0.55,-0.18)$);

\fill[violet!70] ($(unet.center)+(0.95, 0.42)$) rectangle ($(unet.center)+(1.28,-0.42)$);
\fill[violet!40] ($(unet.center)+(1.48, 0.72)$) rectangle ($(unet.center)+(1.86,-0.72)$);
\fill[violet!25] ($(unet.center)+(2.00, 0.95)$) rectangle ($(unet.center)+(2.36,-0.95)$);


\draw[violet!45, line width=0.9pt, line cap=round]
    ($(unet.center)+(-1.30, 0.82)$) -- ($(unet.center)+(2.06, 0.82)$);

\draw[violet!45, line width=0.9pt, line cap=round]
    ($(unet.center)+(-0.62, 0.55)$) -- ($(unet.center)+(1.54, 0.55)$);

\draw[violet!45, line width=0.9pt, line cap=round]
    ($(unet.center)+(-0.02, 0.30)$) -- ($(unet.center)+(1.01, 0.30)$);



\coordinate (tapE) at ($(unet.center)+(-0.82,-0.72)$);
\coordinate (tapD) at ($(unet.center)+(1.67,-0.72)$);

\draw[darr, draw=black!55] (tapE) -- (enc.north);
\draw[darr, draw=black!55] (tapD) -- (dec.north);

\begin{scope}[on background layer]
    \node[
        draw=magenta!55, fill=magenta!4,
        rounded corners=4mm, line width=1.0pt,
        inner xsep=10pt, inner ysep=11pt,
        fit=(meta)(corr)(unet)(x0)(enc)(dec)
    ] (grpProbe) {};

    \node[
        draw=cyan!75, fill=cyan!4,
        rounded corners=4mm, line width=1.0pt,
        inner xsep=10pt, inner ysep=11pt,
        fit=(slotfeat)(pool)(head)(ood)
    ] (grpScore) {};
\end{scope}

\node[font=\bfseries\footnotesize, text=black!72]
    at ($(grpProbe.north)+(0,0.38)$) {canonical probing};

\node[font=\bfseries\footnotesize, text=black!72]
    at ($(grpScore.north)+(0,0.38)$) {lightweight ID-only scoring};

\end{tikzpicture}%
}
\caption{\textbf{Overview of \cfs.} The input is corrupted at a canonical level, processed by a frozen diffusion U-Net, probed through a small number of native internal snapshots, and scored with a lightweight ID-only head.}
\label{fig:cfs_overview}
\end{figure}

\subsection{A local diagnostic view}

Our theory is local and diagnostic rather than universal: it asks which frozen internal statistics should be most useful for testing membership in the evaluation ID reference distribution, and it focuses on quantities directly measurable on held-out data.

Fix a canonical level \(\lambda\). Let \(\mathbf z_{d,\lambda}(\mathbf x_0)\) and \(\mathbf z_{e,\lambda}(\mathbf x_0)\) denote a selected late decoder snapshot and a selected deep encoder snapshot, and define
\[
\mathbf z_\lambda(\mathbf x_0)
=
[\mathbf z_{d,\lambda}(\mathbf x_0)^\top,
 \mathbf z_{e,\lambda}(\mathbf x_0)^\top]^\top .
\]
We use the local testing approximation
\begin{equation}
\mathbf z_\lambda(\mathbf x_0)\mid H_0
\sim
\mathcal N(\boldsymbol{\mu}_\lambda,\boldsymbol{\Sigma}_\lambda),
\qquad
\mathbf z_\lambda(\mathbf x_0)\mid H_1
\sim
\mathcal N(\boldsymbol{\mu}_\lambda+\boldsymbol{\Delta}_\lambda,\boldsymbol{\Sigma}_\lambda),
\label{eq:local_model_main}
\end{equation}
where \(H_0:\mathbf x_0\sim P\) and \(H_1:\mathbf x_0\sim Q\). This is a local model on pooled internal representations, not a global image model.

\begin{theorem}[Conditional encoder-decoder complementarity]
\label{thm:conditional_complementarity_main}
Under Eq.~\eqref{eq:local_model_main}, the paired separation decomposes as
\[
\mathrm{Sep}_{\mathrm{pair}}(\lambda)
=
\mathrm{Sep}_{\mathrm{dec}}(\lambda)
+
\mathrm{Res}_{e\mid d}(\lambda),
\qquad
\mathrm{Res}_{e\mid d}(\lambda)\ge 0.
\]
The residual vanishes iff the encoder shift is fully explained by the decoder shift. Hence
\[
\mathrm{Sep}_{\mathrm{pair}}(\lambda)\ge \mathrm{Sep}_{\mathrm{dec}}(\lambda).
\]
The full expression is given in Appendix~\ref{app:proof_conditional_complementarity}.
\end{theorem}

Thus, the local model supports treating the late decoder as the primary sparse probe, while viewing the encoder as a complementary source of residual information.

\begin{proposition}[Low-noise corruption stability]
\label{prop:low_noise_stability_main}
Let
\[
\mathbf x_\lambda=a(\lambda)\mathbf x_0+b(\lambda)\boldsymbol{\varepsilon},
\qquad
\boldsymbol{\varepsilon}\sim\mathcal N(\mathbf 0,\mathbf I),
\qquad
\mathbf z_{\lambda,h}(\mathbf x_0)
=
\boldsymbol{\phi}_{\lambda,h}(\mathbf x_\lambda).
\]
Assume that \(\boldsymbol{\phi}_{\lambda,h}\) admits a first-order mean-square expansion around
\(a(\lambda)\mathbf x_0\), with Jacobian
\[
\mathbf J_{\lambda,h}(\mathbf x_0)
:=
\nabla_{\mathbf x}\boldsymbol{\phi}_{\lambda,h}(\mathbf x)
\big|_{\mathbf x=a(\lambda)\mathbf x_0}.
\]
Then, conditionally on \(\mathbf x_0\),
\begin{equation}
\mathbb E
\!\left[
\|
\mathbf z_{\lambda,h}(\mathbf x_0)
-
\boldsymbol{\phi}_{\lambda,h}(a(\lambda)\mathbf x_0)
\|_2^2
\,\middle|\,
\mathbf x_0
\right]
=
b(\lambda)^2
\operatorname{tr}
\!\left(
\mathbf J_{\lambda,h}(\mathbf x_0)
\mathbf J_{\lambda,h}(\mathbf x_0)^\top
\right)
+
o(b(\lambda)^2).
\label{eq:low_noise_stability_main}
\end{equation}
Thus, the within-image corruption variance of the hooked representation is first-order proportional to \(b(\lambda)^2\).
\end{proposition}

Finally, Appendix~\ref{app:theory_deferred_results} shows that the oracle diagonal score has mean \(1\) under \(H_0\) and \(1+\kappa_\lambda/d_\lambda\) under \(H_1\), where
\[
\mathbf D_\lambda=\operatorname{diag}(\boldsymbol{\Sigma}_\lambda),
\qquad
\kappa_\lambda
=
\boldsymbol{\Delta}_\lambda^\top
\mathbf D_\lambda^{-1}
\boldsymbol{\Delta}_\lambda,
\qquad
d_\lambda=\dim(\mathbf z_\lambda).
\]
These results motivate two diagnostics:
\[
\hat\kappa_\lambda(\mathcal S),
\qquad
\hat R_h(\lambda),
\]
corresponding to the estimated diagonal separation and content-to-instability ratio. Their operational estimators are introduced in the next subsection, while their formal analysis and supplementary interpretation are deferred to Appendix~\ref{app:theory}.

\subsection{Hook selection and lightweight scoring}

We hook \emph{block outputs} rather than arbitrary leaf submodules. This improves portability across architectures and avoids unstable low-level activations. The local theory suggests two priorities: late decoder states are the primary candidates, and a deep encoder hook, when used, should be interpreted as complementary residual information.

Within each structural region, we refine the final choice with a small ID-only proxy. For candidate module \(m\), let \(\mathbf z_{i,r}^{(m)}\) be its pooled feature for image \(i\) under corruption draw \(r\), averaged over the selected canonical levels, and let
\[
\bar{\mathbf z}_i^{(m)}
=
\frac{1}{R}\sum_{r=1}^R \mathbf z_{i,r}^{(m)}.
\]
We estimate
\[
\widehat{\mathbf C}_{\mathrm{img}}^{(m)}
=
\widehat{\operatorname{Cov}}_i
(\bar{\mathbf z}_i^{(m)}),
\qquad
\widehat{\mathbf C}_{\mathrm{corr}}^{(m)}
=
N^{-1}\sum_{i=1}^N
\widehat{\operatorname{Cov}}_r
(\mathbf z_{i,r}^{(m)}),
\]
and score the candidate by
\begin{equation}
\mathrm{Proxy}(m)
=
\frac{
\operatorname{tr}\widehat{\mathbf C}_{\mathrm{img}}^{(m)}
}{
\operatorname{tr}\widehat{\mathbf C}_{\mathrm{corr}}^{(m)}
}.
\label{eq:hook_proxy_main}
\end{equation}
This favors modules that vary across ID images while remaining stable under stochastic corruption at fixed content.

For a selected slot \((k,\ell)\), let
\[
\mathbf F_{k,\ell}(\mathbf x_0)
\in
\mathbb R^{C_{k,\ell}\times H_{k,\ell}\times W_{k,\ell}}
\]
denote the hooked feature map. We use the pooled descriptor
\begin{equation}
\mathbf z_{k,\ell}(\mathbf x_0)
=
[
\operatorname{Mean}_{\mathrm{sp}}(\mathbf F_{k,\ell}(\mathbf x_0))^\top,
\operatorname{Std}_{\mathrm{sp}}(\mathbf F_{k,\ell}(\mathbf x_0))^\top
]^\top
\in
\mathbb R^{D_{k,\ell}},
\label{eq:pooled_feature_new}
\end{equation}
where \(\operatorname{Mean}_{\mathrm{sp}}\) and \(\operatorname{Std}_{\mathrm{sp}}\) are channel-wise spatial mean and standard deviation. Using only ID-train data, we fit diagonal statistics
\((\hat{\boldsymbol{\mu}}_{k,\ell},\hat{\mathbf v}_{k,\ell})\). The slot score is
\begin{equation}
s_{k,\ell}(\mathbf x_0)
=
\frac{1}{D_{k,\ell}}
\sum_{j=1}^{D_{k,\ell}}
\frac{
\left(
\mathbf z_{k,\ell}^{(j)}(\mathbf x_0)
-
\hat{\boldsymbol{\mu}}_{k,\ell}^{(j)}
\right)^2
}{
\hat{\mathbf v}_{k,\ell}^{(j)}
},
\label{eq:slot_score_new}
\end{equation}
and the final score is
\begin{equation}
S_{\cfs}(\mathbf x_0)
=
\frac{1}{|\mathcal S|}
\sum_{(k,\ell)\in\mathcal S}
s_{k,\ell}(\mathbf x_0).
\label{eq:cfs_score_new}
\end{equation}

We intentionally use a lightweight diagonal score, also referred to as an ID-only score or ID-only proxy: \cfs\ is meant to test the quality of sparse frozen representations, not to rely on a high-capacity downstream classifier. While stronger downstream heads may improve absolute performance, they would partially confound the representation-quality interpretation targeted by \mbe. Alternative heads are therefore reported only as analysis in Appendix~\ref{app:ablation_heads}. For the local diagnostics, we estimate
\[
\hat\kappa_{\lambda}(\mathcal S)
=
\hat{\boldsymbol{\Delta}}_{\lambda,\mathcal S}^{\top}
\hat{\mathbf D}_{\lambda,\mathcal S}^{-1}
\hat{\boldsymbol{\Delta}}_{\lambda,\mathcal S},
\qquad
\hat R_h(\lambda)
=
\frac{
\operatorname{tr}\widehat{\mathbf C}_{\mathrm{img}}^{(h,\lambda)}
}{
\operatorname{tr}\widehat{\mathbf C}_{\mathrm{corr}}^{(h,\lambda)}
}.
\]
These diagnostics are not used to fit the detector; they only test whether sparse hooks and canonical levels behave as predicted by the local theory.

\section{Experiments}
\label{sec:experiments}

\subsection{Setup}

Our main benchmark is CIFAR-scale, where protocol control is the cleanest. We use
\begin{align}
\mathcal{I}_{\mathrm{small}}& =\{\text{CIFAR-10},\text{SVHN},\text{CelebA32}\}\, , \nonumber \\
\mathcal{O}_{\mathrm{small}}& =\{\text{CIFAR-10},\text{SVHN},\text{CelebA32},\text{CIFAR-100},\text{DTD}\}\,. \nonumber
\end{align}
For each \(i\in\mathcal{I}_{\mathrm{small}}\), all datasets in \(\mathcal{O}_{\mathrm{small}}\setminus\{i\}\) are treated as OOD, yielding \(12\) ID\(\to\)OOD pairs per backbone.

We evaluate two diffusion backbone families through the same adapter and canonical corruption interface: improved-diffusion backbones with discrete timesteps and native \(\eps\)-prediction, and EDM-family backbones with continuous noise conditioning and one-shot \(\hat x_0\) estimation. In the main benchmark, all methods use a shared CIFAR-10 checkpoint family; we then repeat the same benchmark with a shared CelebA32 checkpoint family for source-family robustness. Larger-scale ImageNet transfer results are deferred to the Appendix~\ref{app:large_scale}.

\subsection{Baselines and metrics}

All baselines are evaluated under the shared \mbe\ pipeline. Our main comparators are \msma\ \citep{mahmood2021msma}, \diffpath\ \citep{heng2024diffpath}, \ddpmood\ \citep{graham2023ddpmodd}, and \gepc\ \citep{rouzoumka2026gepc}.  Detailed baseline taxonomy and implementation choices are deferred to Appendix~\ref{app:protocols} and Appendix~\ref{app:taxonomy_baselines}. 



Unless stated otherwise, the main paper operating point is \(\cfs(1{\times}2)\), a one-level two-hook variant using paired encoder and decoder snapshots at the same low-noise canonical level. We also report \(\cfs_{\mathrm{dec}}(1{\times}1)\), a decoder-only companion. Richer variants such as \(\cfs(2\times2)\) and \(\cfs(2\times4)\) are studied in the Appendix~\ref{app:ablation_budget}.

All methods output OOD-high scores. Our primary metric is AUROC; full FPR95 breakdowns are reported in the appendix. Across pairs, we summarize performance by
\begin{align}
\AvgAUROC_{m,b}
&=
\frac{1}{|\mathcal{P}|}\sum_{p\in\mathcal{P}}\mathrm{AUROC}_{m,b}(p),\\
\AvgWorstAUROC_{m,b}
&=
\frac{1}{|\mathcal{I}|}\sum_{i\in\mathcal{I}}
\min_{o\in\mathcal{O}(i)}
\mathrm{AUROC}_{m,b}(i\to o),
\end{align}
and report logical test-time \(\Cost_m=\#F_m+\#J_m\) (with \(\#J_m=0\) for all main-paper comparisons).


Appendix~\ref{app:arch_transfer} provides an architecture-transfer sanity check on U-ViT, where CFS is applied to early/middle/late transformer block snapshots rather than U-Net encoder/decoder maps.


\section{Results}
\label{sec:results}




\textbf{Claim scope.}
Our claim is strictly about controlled comparison under \mbe: when source-family policy, canonical corruption semantics, and logical test-time cost are matched, sparse native snapshot probing is stronger than the output-space alternatives we evaluate at comparable or lower logical cost.

\subsection{Main CIFAR-scale comparison under \mbe}

Table~\ref{tab:main_cifar_ckpt} reports the main controlled CIFAR-scale comparison under \mbe.

\begin{table*}[t]
\centering
\caption{\textbf{Main comparison under \mbe\ on the CIFAR-scale benchmark.} IDs: CIFAR-10, SVHN, CelebA32. OODs: the remaining datasets among \{CIFAR-10, SVHN, CelebA32, CIFAR-100, DTD\}, yielding \(12\) pairs per backbone. All methods use the same source-family policy, preprocessing, canonical corruption construction, and split policy. Best and second-best results are highlighted in bold and underline, respectively. Complementary results are available in Appendix~\ref{app:cifar}.}
\label{tab:main_cifar_ckpt}
\small
\setlength{\tabcolsep}{4pt}
\resizebox{\textwidth}{!}{%
\begin{tabular}{lccccc}
\toprule
Method & Backbone & \(\AvgAUROC \uparrow\) & \(\AvgWorstAUROC \uparrow\) & \(\#F\)/img \(\downarrow\) & Notes \\
\midrule
\msma\ & improved & 0.792 & 0.688 & 10  & multiscale output descriptor \\
\msma\ & EDM      & 0.796 & 0.689 & 10  & multiscale output descriptor \\
\addlinespace
\diffpath\ & improved & 0.778 & 0.641 & 10  & recursive path statistic \\
\diffpath\ & EDM      & 0.792 & 0.635 & 10  & recursive path statistic \\
\addlinespace
\ddpmood\ & improved & 0.550 & 0.316 & 364 & reconstruction / manifold \\
\ddpmood\ & EDM      & 0.559 & 0.320 & 364 & reconstruction / manifold \\
\addlinespace
\gepc\ & improved & 0.616 & 0.546 & \second{8} & output-space consistency \\
\gepc\ & EDM      & 0.774 & 0.600 & \second{8} & output-space consistency \\
\midrule
\textbf{\(\cfs_{\mathrm{dec}}(1{\times}1)\)} & improved & \second{0.886} & \second{0.793} & \best{1} & single snapshot at low noise \\
\textbf{\(\cfs_{\mathrm{dec}}(1{\times}1)\)} & EDM      & \best{0.919} & \second{0.809} & \best{1} & single snapshot at low noise \\
\textbf{\(\cfs(1{\times}2)\)}     & improved & \best{0.887} & \best{0.799} & \best{1} & paired snapshots at low noise \\
\textbf{\(\cfs(1{\times}2)\)}     & EDM      & \second{0.916} & \best{0.814} & \best{1} & paired snapshots at low noise \\
\bottomrule
\end{tabular}}
\end{table*}


At a one-forward budget, \(\cfs(1{\times}2)\) is the strongest operating point on both backbone families. Relative to the strongest harmonized output-space baselines, it improves AvgWorstAUROC from \(0.689\) to \(0.814\) on EDM and from \(0.688\) to \(0.799\) on improved-diffusion, while reducing logical cost from \(8\)--\(10\) forwards per image to \(1\). At the same time, \(\cfs_{\mathrm{dec}}(1{\times}1)\) remains extremely close, showing that the representation-space gain is not only strong but also highly compressible. Budget-matched ablations are reported in Appendix~\ref{app:ablation_budget}.

DLSR is the closest prior in spirit, but it lies outside our MBE protocol because it introduces an additional learned feature-reconstruction module. Since DLSR is only available on its native published evaluation pairs, we report a comparison on that DLSR-native subset in Table~\ref{tab:dlsr_overlap} of Appendix~\ref{app:dlsr_positioning}.

\subsection{Theory diagnostics: measurable quantities predict sparse-probe quality}
\label{sec:theory_diagnostics_results}

We next test whether the local-testing theory yields measurable quantities that predict sparse-probe quality. It does (see Table~\ref{tab:theory_diag_summary}).

First, the diagonal noncentrality diagnostic \(\hat\kappa_\lambda(\mathcal S)/d\) is strongly aligned with downstream OOD performance across both backbone families: candidate sparse probes with larger estimated diagonal separation consistently yield larger AUROC. This supports the interpretation of the diagonal score as a local detector. Representative diagnostic scatter plots for \(\hat\kappa_\lambda(\mathcal S)/d\) versus AUROC are deferred to Appendix~\ref{app:empirical_diagnostic_protocol}. 

Second, the content-to-instability ratio \(\hat R_h(\lambda)\) is also strongly predictive of hook quality. Moreover, the within-image corruption variance decreases sharply as \(b(\lambda)^2\) decreases, in agreement with Proposition~\ref{prop:low_noise_stability_main}, supporting the low-noise advantage. Appendix~\ref{app:empirical_diagnostic_protocol} reports the complementary \(\hat R_h(\lambda)\) and low-noise stability plots.


\begin{table}[t]
\centering
\caption{
\textbf{Empirical validation of the measurable theory diagnostics.}
Spearman correlations are computed across candidate sparse probes and ID\(\to\)OOD pairs.
The diagnostics are not used to fit the detector; they only test whether the measured behavior of sparse hooks and canonical levels agrees with the local theory.
}
\label{tab:theory_diag_summary}
\small
\setlength{\tabcolsep}{4pt}
\begin{tabular}{llcc}
\toprule
Diagnostic & Backbone & Spearman \(\rho\)\\ 
\midrule
\(\hat\kappa_\lambda(\mathcal S)/d\) vs AUROC & improved & 0.923 \\ 
\(\hat\kappa_\lambda(\mathcal S)/d\) vs AUROC & EDM      & 0.938 \\ 
\addlinespace
\(\hat R_h(\lambda)\) vs Avg AUROC            & improved & 0.862 \\ 
\(\hat R_h(\lambda)\) vs Avg AUROC            & EDM      & 0.760 \\ 
\addlinespace
\bottomrule
\end{tabular}
\end{table}

\subsection{Focused ablations}


The appendix addresses three narrower questions: whether the gain is explained by hidden test-time budget in Appendix~\ref{app:ablation_budget}, whether the gain is driven by head complexity rather than representation quality (see Appendix~\ref{app:ablation_repr_comp}), and whether the ranking is stable across stochastic seeds through Appendix~\ref{app:ablation_seeds}. Across all cases, the central conclusion is unchanged: the gain comes primarily from \emph{where} the frozen backbone is probed.

\noindent\textbf{Hook-selection robustness.}
The ID-only hook-selection proxy does not need to recover the exact oracle pair to be effective. On improved-diffusion, the selected pair reaches \(0.886\) AvgAUROC versus \(0.895\) for the best admissible pair, while on EDM, the pairwise landscape exhibits a broad plateau of strong pairs. This indicates that \(\cfs(1{\times}2)\) is not driven by brittle hook cherry-picking; see Appendix~\ref{app:ablation_hooks}.

\noindent\textbf{External positioning outside \mbe.}
Table~\ref{tab:external_positioning_mini} provides a non-protocol-matched positioning against previously reported diffusion-OOD results. These comparisons are included for external positioning. 

\begin{table}[t]
\centering
\caption{\textbf{External positioning outside \mbe\ (single-checkpoint setting only).} Prior-method results are taken from the respective papers or officially reported artifacts.
Full results are in Appendix~\ref{app:external_positioning}.}
\label{tab:external_positioning_mini}
\small
\begin{tabular}{lcc}
\toprule
Method & Avg. AUROC $\uparrow$ & Cost $\downarrow$ \\
\midrule
SCOPED-CelebA & 0.892 & 2F+2J \\
GEPC-CelebA & 0.910 & 8F \\
DiffPath-6D-CelebA & 0.931 & 10F \\
DiffPath-6D-ImageNet & 0.850 & 10F \\
\textbf{\(\cfs(1{\times}2)\)-CelebA (ours)} & \second{0.935} & \best{1F} \\
\textbf{\(\cfs(1{\times}2)\)-ImageNet (ours)} & \best{0.962} & \best{1F} \\
\bottomrule
\end{tabular}
\end{table}

\subsection{Source-family robustness}

In Table~\ref{tab:source_family_summary}, we repeat the same CIFAR-scale benchmark with a shared CelebA32 source family to test whether the ranking persists after changing the frozen source representation.

\begin{table}[t]
\centering
\caption{\textbf{Source-family robustness on the CIFAR-scale benchmark.} Same protocol as Table~\ref{tab:main_cifar_ckpt}, but with a shared CelebA32 checkpoint family.  More experiments are exposed in Appendix~\ref{app:source_family_full}.}
\label{tab:source_family_summary}
\footnotesize
\setlength{\tabcolsep}{3.4pt}
\renewcommand{\arraystretch}{1.03}
\begin{tabular}{lccccc}
\toprule
& \multicolumn{2}{c}{Improved} & \multicolumn{2}{c}{EDM} & \multirow{2}{*}{\(\#F\)/img \(\downarrow\)} \\
\cmidrule(lr){2-3}\cmidrule(lr){4-5}
Method & \(\AvgAUROC \uparrow\) & \(\AvgWorstAUROC \uparrow\) & \(\AvgAUROC \uparrow\) & \(\AvgWorstAUROC \uparrow\) & \\
\midrule
\msma\                  & 0.881 & 0.808 & 0.790 & 0.673 & 10 \\
\diffpath\              & 0.829 & 0.743 & 0.776 & 0.650 & 10 \\
\ddpmood\               & 0.581 & 0.362 & 0.575 & 0.325 & 364 \\
\gepc\                  & 0.749 & 0.650 & 0.745 & 0.648 & \second{8} \\
\midrule
\textbf{\(\cfs_{\mathrm{dec}}(1{\times}1)\)} & \second{0.907} & \second{0.823} & \second{0.914} & \second{0.827} & \best{1} \\
\textbf{\(\cfs(1{\times}2)\)}     & \best{0.908}   & \best{0.827}   & \best{0.928}   & \best{0.846}   & \best{1} \\
\bottomrule
\end{tabular}
\end{table}

The ordering remains essentially unchanged under both source families: \(\cfs(1{\times}2)\) remains best on both backbone families, and \(\cfs_{\mathrm{dec}}(1{\times}1)\) remains a near-tied one-forward companion. This argues against a checkpoint-family artifact and supports the probing-space interpretation. 

\section{Discussion and Limitations}
\label{sec:discussion}

Once protocol confounders are removed, the main bottleneck is not simply \emph{more output-space engineering}, but \emph{where} the frozen diffusion backbone is probed. Our results suggest that a small number of native internal snapshots can retain relative-OOD geometry that output-space summaries partly attenuate. In this sense, \cfs\ should be read not only as a score family, but as evidence that, under controlled \mbe\ comparison, a representation-first probing strategy can be more informative than broader output-space summaries.
Across controlled comparisons, \(\cfs(1{\times}2)\) is the strongest one-forward operating point, while the sparse \(\cfs_{\mathrm{dec}}(1{\times}1)\) remains remarkably competitive, showing that a large fraction of the useful signal is already concentrated in a single late decoder snapshot.

\textbf{Limitations.}
\cfs\ requires internal access to the diffusion backbone and is therefore less black-box than output-space scores. Its portability is structural rather than representational: canonical levels can be aligned across backbones, but internal features remain architecture and source-dependent. Hard multimodal ID regimes such as CIFAR-10 remain challenging. Finally, our theory is intentionally local rather than universal. 


\section{Conclusion}
\label{sec:conclusion}
We introduced \mbe, a protocol for fair cross-backbone diffusion OOD evaluation, and \cfs, a family of minimal representation-space detectors based on sparse native internal snapshots. Under \mbe, the strongest one-forward operating point in our main controlled comparisons is \(\cfs(1{\times}2)\), while the even sparser \(\cfs_{\mathrm{dec}}(1{\times}1)\) remains remarkably competitive. More broadly, once protocol confounders are removed, strong relative diffusion OOD detection can already be obtained far upstream of elaborate output-space probes: a tiny number of sparse, canonically aligned internal snapshots captures a large fraction of the useful relative-OOD geometry.

\bibliographystyle{plainnat}
\bibliography{refs}

\appendix

\section{Appendix Roadmap}
\label{app:roadmap}

This appendix supports the paper's central claim: under a backbone-equated protocol, a frozen diffusion checkpoint can be used as a relative-OOD representation map, and a tiny number of sparse low-noise internal snapshots already captures strong useful signal.

It is organized as follows:
\begin{itemize}
\item Section~\ref{app:theory}: proofs and supplementary validations for the main-paper theory, including conditional encoder-decoder complementarity, low-noise corruption stability, diagonal-score separation, canonical-level matching, and the interpretation of the hook-selection proxy.
\item Sections~\ref{app:protocols}--\ref{app:taxonomy_baselines}: shared protocol, canonicalization, logical cost accounting, and harmonized baseline implementations.
\item Section~\ref{app:ablations}: focused ablations testing budget, hook robustness, canonical-level robustness, pooling, head choice, bank size, and seed stability.
\item Sections~\ref{app:cifar} and \ref{app:source_family_full}: full CIFAR-scale pairwise results under the primary and alternative source-family policies.
\item Section~\ref{app:external_positioning}: external positioning against prior reported diffusion-based CIFAR-scale results outside the controlled \mbe\ protocol.
\item Section~\ref{app:large_scale}: checkpoint-controlled large-scale results on ImageNet200 and ImageNet1K using a single official ImageNet-64 improved-diffusion backbone.
\item Section~\ref{app:impl}: implementation and configuration details, including the main \cfs\ hyperparameters, ID-only head configurations, determinism settings, profiling protocol, and representative compute-resource reporting.
\item Section~\ref{app:broader_impact}: broader impact and responsible-use discussion.
\end{itemize}

Unless noted otherwise, appendix ablations are centered on \cfs{}$(1\times2)$, the primary main-paper operating point. We retain \(\cfs_{\mathrm{dec}}(1{\times}1)\) only in targeted controls where the compression result is directly relevant.



\section{Theory Appendix}
\label{app:theory}

This appendix supplies the formal statements, proofs, and supplementary validations supporting the main paper theory. The main text relies directly on two load-bearing claims: conditional encoder-decoder complementarity and low-noise corruption stability. We additionally formalize here the diagonal-score separation result deferred from the main text,
a canonical-level stability bound under discretization, and a formal interpretation of the hook-selection proxy.

\subsection{Notation}
\label{app:theory_notation}

Fix a frozen diffusion backbone and evaluation domains \(P\) and \(Q\), corresponding to the
relative ID and OOD distributions. For a canonical corruption level \(\lambda\),
\[
\mathbf x_\lambda
=
a(\lambda)\, \mathbf x_0
+
b(\lambda)\,\boldsymbol{\varepsilon}\,,
\qquad
\boldsymbol{\varepsilon}\sim\mathcal N(\mathbf 0,\mathbf I)\,.
\]
For a selected late decoder hook and a selected deep encoder hook,
\[
\mathbf z_\lambda(\mathbf x_0)
=
\begin{bmatrix}
\mathbf z_{d,\lambda}(\mathbf x_0)\\
\mathbf z_{e,\lambda}(\mathbf x_0)
\end{bmatrix}\,,
\]
and the local model is
\[
\mathbf z_\lambda(\mathbf x_0)\mid H_0
\sim
\mathcal N(\boldsymbol{\mu}_\lambda,\boldsymbol{\Sigma}_\lambda)\,,
\qquad
\mathbf z_\lambda(\mathbf x_0)\mid H_1
\sim
\mathcal N(\boldsymbol{\mu}_\lambda+\boldsymbol{\Delta}_\lambda,\boldsymbol{\Sigma}_\lambda)\,,
\]
with \(H_0:\mathbf x_0\sim P\) and \(H_1:\mathbf x_0\sim Q\).

\subsection{Deferred statement from the main text}
\label{app:theory_deferred_results}

The following result is used in the interpretation of the diagonal \cfs\ score, but is deferred here to keep the main text focused.

Let
\[
\mathbf D_\lambda=\operatorname{diag}(\boldsymbol{\Sigma}_\lambda),
\qquad
d_\lambda=\dim(\mathbf z_\lambda),
\]
and consider the oracle diagonal score
\[
s^\circ_\lambda(\mathbf z)
:=
\frac{1}{d_\lambda}\,
(\mathbf z-\boldsymbol{\mu}_\lambda)^\top\,
\mathbf D_\lambda^{-1}\,
(\mathbf z-\boldsymbol{\mu}_\lambda)\,,
\qquad
\kappa_\lambda
:=
\boldsymbol{\Delta}_\lambda^\top\,
\mathbf D_\lambda^{-1}\,
\boldsymbol{\Delta}_\lambda\,.
\]

\begin{theorem}[Diagonal-score separation under the local model]
\label{thm:diag_score_app}
Under Eq.~\eqref{eq:local_model_main},
\[
\mathbb E_{H_0}[s^\circ_\lambda]=1\,,
\qquad
\mathbb E_{H_1}[s^\circ_\lambda]=1+\frac{\kappa_\lambda}{d_\lambda}\,.
\]
Moreover, the variance of \(s^\circ_\lambda\) is controlled by the correlation structure of
\(\boldsymbol{\Sigma}_\lambda\), yielding the detection-power bound proved below.
\end{theorem}

\subsection{Proof of Theorem~\ref{thm:conditional_complementarity_main}}
\label{app:proof_conditional_complementarity}

\begin{proof}
Write
\[
\boldsymbol{\Delta}_\lambda
=
\begin{bmatrix}
\boldsymbol{\Delta}_{d,\lambda}\\
\boldsymbol{\Delta}_{e,\lambda}
\end{bmatrix}\,,
\qquad
\boldsymbol{\Sigma}_\lambda
=
\begin{bmatrix}
\boldsymbol{\Sigma}_{dd,\lambda} & \boldsymbol{\Sigma}_{de,\lambda}\\
\boldsymbol{\Sigma}_{ed,\lambda} & \boldsymbol{\Sigma}_{ee,\lambda}
\end{bmatrix}\,.
\]
The paired Mahalanobis separation is
\[
\mathrm{Sep}_{\mathrm{pair}}(\lambda)
=
\boldsymbol{\Delta}_\lambda^\top\,
\boldsymbol{\Sigma}_\lambda^{-1}\,
\boldsymbol{\Delta}_\lambda\,,
\]
while the decoder-only separation is
\[
\mathrm{Sep}_{\mathrm{dec}}(\lambda)
=
\boldsymbol{\Delta}_{d,\lambda}^\top\,
\boldsymbol{\Sigma}_{dd,\lambda}^{-1}\,
\boldsymbol{\Delta}_{d,\lambda}\,.
\]
Let
\[
\boldsymbol{\Sigma}_{e\mid d,\lambda}
=
\boldsymbol{\Sigma}_{ee,\lambda}
-
\boldsymbol{\Sigma}_{ed,\lambda}\,
\boldsymbol{\Sigma}_{dd,\lambda}^{-1}\,
\boldsymbol{\Sigma}_{de,\lambda}\, ,
\]
be the Schur complement of the decoder block. By the block inverse formula,
\[
\boldsymbol{\Sigma}_\lambda^{-1}
=
\begin{bmatrix}
\boldsymbol{\Sigma}_{dd,\lambda}^{-1}
+
\boldsymbol{\Sigma}_{dd,\lambda}^{-1}\,
\boldsymbol{\Sigma}_{de,\lambda}\,
\boldsymbol{\Sigma}_{e\mid d,\lambda}^{-1}\,
\boldsymbol{\Sigma}_{ed,\lambda}\,
\boldsymbol{\Sigma}_{dd,\lambda}^{-1}
&
-
\boldsymbol{\Sigma}_{dd,\lambda}^{-1}\,
\boldsymbol{\Sigma}_{de,\lambda}\,
\boldsymbol{\Sigma}_{e\mid d,\lambda}^{-1}
\\[0.25em]
-
\boldsymbol{\Sigma}_{e\mid d,\lambda}^{-1}\,
\boldsymbol{\Sigma}_{ed,\lambda}\,
\boldsymbol{\Sigma}_{dd,\lambda}^{-1}
&
\boldsymbol{\Sigma}_{e\mid d,\lambda}^{-1}
\end{bmatrix}\,.
\]
Substituting this expression into
\(\boldsymbol{\Delta}_\lambda^\top\,
\boldsymbol{\Sigma}_\lambda^{-1}\,
\boldsymbol{\Delta}_\lambda\)
and collecting terms gives
\[
\mathrm{Sep}_{\mathrm{pair}}(\lambda)
=
\mathrm{Sep}_{\mathrm{dec}}(\lambda)
+
\mathrm{Res}_{e\mid d}(\lambda)\, ,
\]
where
\[
\mathrm{Res}_{e\mid d}(\lambda)
=
\left(
\boldsymbol{\Delta}_{e,\lambda}
-
\boldsymbol{\Sigma}_{ed,\lambda}\,
\boldsymbol{\Sigma}_{dd,\lambda}^{-1}\,
\boldsymbol{\Delta}_{d,\lambda}
\right)^\top\,
\boldsymbol{\Sigma}_{e\mid d,\lambda}^{-1}\,
\left(
\boldsymbol{\Delta}_{e,\lambda}
-
\boldsymbol{\Sigma}_{ed,\lambda}\,
\boldsymbol{\Sigma}_{dd,\lambda}^{-1}\,
\boldsymbol{\Delta}_{d,\lambda}
\right)\,.
\]
Since
\(\boldsymbol{\Sigma}_{e\mid d,\lambda}\succ 0\),
the residual term is nonnegative. Therefore
\[
\mathrm{Sep}_{\mathrm{pair}}(\lambda)
\ge
\mathrm{Sep}_{\mathrm{dec}}(\lambda)\,.
\]
Equality holds if and only if the conditional residual vanishes, namely
\[
\boldsymbol{\Delta}_{e,\lambda}
=
\boldsymbol{\Sigma}_{ed,\lambda}\,
\boldsymbol{\Sigma}_{dd,\lambda}^{-1}\,
\boldsymbol{\Delta}_{d,\lambda}\,.
\]
This proves the decomposition and the claimed nonnegativity.
\end{proof}

\subsection{Proof of Proposition~\ref{prop:low_noise_stability_main}}
\label{app:proof_low_noise_stability}

\begin{proof}
Fix \(\mathbf x_0\) and write
\[
\mathbf x_\lambda
=
a(\lambda)\, \mathbf x_0
+
b(\lambda)\,\boldsymbol{\varepsilon}\, .
\]
By the assumed first-order mean-square expansion of
\(\boldsymbol{\phi}_{\lambda,h}\) around
\(a(\lambda)\mathbf x_0\), we have
\[
\boldsymbol{\phi}_{\lambda,h}(\mathbf x_\lambda)
=
\boldsymbol{\phi}_{\lambda,h}(a(\lambda)\mathbf x_0)
+
b(\lambda)\, 
\mathbf J_{\lambda,h}(\mathbf x_0)\, 
\boldsymbol{\varepsilon}
+
\mathbf r_{\lambda,h}(\mathbf x_0,\boldsymbol{\varepsilon})\, ,
\]
with
\[
\mathbb E
\!\left[
\left\|\mathbf r_{\lambda,h}(\mathbf x_0,\boldsymbol{\varepsilon})\right\|_2^2
\,\middle|\,
\mathbf x_0
\right]
=
o\left(b(\lambda)^2\right)\, .
\]
Therefore,
\[
\mathbf z_{\lambda,h}(\mathbf x_0)
-
\boldsymbol{\phi}_{\lambda,h}(a(\lambda)\, \mathbf x_0)
=
b(\lambda)\, 
\mathbf J_{\lambda,h}(\mathbf x_0)\, 
\boldsymbol{\varepsilon}
+
\mathbf r_{\lambda,h}(\mathbf x_0,\boldsymbol{\varepsilon})\,.
\]
Taking squared norms and conditional expectations gives
\[
\mathbb E
\!\left[
\left\|
\mathbf z_{\lambda,h}(\mathbf x_0)
-
\boldsymbol{\phi}_{\lambda,h}(a(\lambda)\,\mathbf x_0)
\right\|_2^2
\,\middle|\,
\mathbf x_0
\right]
=
b(\lambda)^2
\mathbb E
\!\left[
\left\|
\mathbf J_{\lambda,h}(\mathbf x_0)
\boldsymbol{\varepsilon}
\right\|_2^2
\right]
+
o\left(b(\lambda)^2\right)\, .
\]
Since
\(\boldsymbol{\varepsilon}\sim\mathcal N(\mathbf 0,\mathbf I)\),
\[
\mathbb E
\!\left[
\left\|
\mathbf J_{\lambda,h}(\mathbf x_0)
\, \boldsymbol{\varepsilon}
\right\|_2^2
\right]
=
\operatorname{tr}
\!\left(
\mathbf J_{\lambda,h}(\mathbf x_0)\, 
\mathbf J_{\lambda,h}(\mathbf x_0)^\top
\right)\, .
\]
Hence
\[
\mathbb E
\!\left[
\left\|
\mathbf z_{\lambda,h}(\mathbf x_0)
-
\boldsymbol{\phi}_{\lambda,h}(a(\lambda)\, \mathbf x_0)
\right\|_2^2
\,\middle|\,
\mathbf x_0
\right]
=
b(\lambda)^2
\operatorname{tr}
\!\left(
\mathbf J_{\lambda,h}(\mathbf x_0)\, 
\mathbf J_{\lambda,h}(\mathbf x_0)^\top
\right)
+
o(b\left(\lambda)^2\right)\,,
\]
which proves the result.
\end{proof}

\subsection{Proof of Theorem~\ref{thm:diag_score_app}}
\label{app:proof_diag_score}

\begin{proof}
Let
\[
\mathbf D_\lambda
=
\operatorname{diag}(\boldsymbol{\Sigma}_\lambda),
\qquad
\mathbf R_\lambda
=
\mathbf D_\lambda^{-1/2}
\boldsymbol{\Sigma}_\lambda
\mathbf D_\lambda^{-1/2},
\qquad
\boldsymbol{\delta}_\lambda
=
\mathbf D_\lambda^{-1/2}
\boldsymbol{\Delta}_\lambda\,.
\]
Define the whitened-but-correlated variable
\[
\mathbf y
=
\mathbf D_\lambda^{-1/2}
(\mathbf z-\boldsymbol{\mu}_\lambda)\,.
\]
Under \(H_0\), we have
\[
\mathbf y
\sim
\mathcal N(\mathbf 0,\mathbf R_\lambda)\,,
\]
and under \(H_1\),
\[
\mathbf y
\sim
\mathcal N(\boldsymbol{\delta}_\lambda,\mathbf R_\lambda)\,.
\]
The oracle diagonal score can be written as
\[
s_\lambda^\circ(\mathbf z)
=
\frac{1}{d_\lambda}\,
\mathbf y^\top \,\mathbf y\,.
\]

Because \(\mathbf R_\lambda\) is a correlation matrix,
\(\operatorname{tr}(\mathbf R_\lambda)=d_\lambda\). Hence
\[
\mathbb E_{H_0}[\mathbf y^\top \mathbf y]
=
\operatorname{tr}(\mathbf R_\lambda)
=
d_\lambda\,,
\]
so
\[
\mathbb E_{H_0}[s_\lambda^\circ]=1\, .
\]
Under \(H_1\),
\[
\mathbb E_{H_1}\left[\mathbf y^\top \mathbf y\right]
=
\operatorname{tr}(\mathbf R_\lambda)
+
\boldsymbol{\delta}_\lambda^\top\,
\boldsymbol{\delta}_\lambda
=
d_\lambda+\kappa_\lambda\, ,
\]
where
\[
\kappa_\lambda
=
\boldsymbol{\Delta}_\lambda^\top\,
\mathbf D_\lambda^{-1}\,
\boldsymbol{\Delta}_\lambda
=
\boldsymbol{\delta}_\lambda^\top\,
\boldsymbol{\delta}_\lambda\, .
\]
Therefore
\[
\mathbb E_{H_1}[s_\lambda^\circ]
=
1+\frac{\kappa_\lambda}{d_\lambda}\,.
\]

For a Gaussian quadratic form \(\mathbf y^\top \,\mathbf A \,\mathbf y\) with
\(\mathbf A=\mathbf I\), the variance satisfies
\[
\operatorname{Var}\left(\mathbf y^\top \, \mathbf y\right)
=
2\operatorname{tr}\left(\mathbf R_\lambda^2\right)
+
4\,\boldsymbol{\delta}_\lambda^\top
\,\mathbf R_\lambda
\,\boldsymbol{\delta}_\lambda\,,
\]
with the second term absent under \(H_0\). Dividing by \(d_\lambda^2\) gives
\[
\operatorname{Var}_{H_0}(s_\lambda^\circ)
=
\frac{2}{d_\lambda^2}\,
\operatorname{tr}\left(\mathbf R_\lambda^2\right)\,,
\]
and
\[
\operatorname{Var}_{H_1}(s_\lambda^\circ)
=
\frac{2}{d_\lambda^2}\,
\operatorname{tr}\left(\mathbf R_\lambda^2\right)
+
\frac{4}{d_\lambda^2}\,
\boldsymbol{\delta}_\lambda^\top\,
\mathbf R_\lambda\,
\boldsymbol{\delta}_\lambda\,.
\]

Finally, apply Cantelli's inequality under \(H_1\) to the random variable
\(s_\lambda^\circ\). For any
\(\tau < \mathbb E_{H_1}[s_\lambda^\circ]\),
\[
\Pr_{H_1}[s_\lambda^\circ \le \tau]
\le
\frac{
\operatorname{Var}_{H_1}(s_\lambda^\circ)
}{
\operatorname{Var}_{H_1}(s_\lambda^\circ)
+
\left(
\mathbb E_{H_1}[s_\lambda^\circ]-\tau
\right)^2
}\, .
\]
Taking complements and substituting the mean gives
\[
\Pr_{H_1}[s_\lambda^\circ > \tau]
\ge
1-
\frac{
\operatorname{Var}_{H_1}(s_\lambda^\circ)
}{
\operatorname{Var}_{H_1}(s_\lambda^\circ)
+
\left(
1+\kappa_\lambda/d_\lambda-\tau
\right)^2
}\, .
\]
This proves the theorem.
\end{proof}

\subsection{Canonical-level matching sanity check}
\label{app:canonical_stability}

This diagnostic tests the implementation role of canonical logSNR matching. For a discrete
backbone, a requested canonical level \(\lambda_{\mathrm{ref}}\) must be mapped to a native timestep
\(t\), whose effective logSNR can be different from \(\lambda_{\mathrm{ref}}\) and is denoted by \(\lambda_t\).

The following simple bound explains why such mismatches can matter: if the internal descriptor and the resulting slot score vary smoothly with the canonical level, then the logSNR mismatch induces
controlled score drift.

Fix a selected hook \(h\). Assume that, on the evaluation domain \(\mathcal X\), the pooled descriptor
\(\mathbf z_{.,h}(\mathbf x)\) is \(L_{h}\)-Lipschitz in the canonical level:
\[
\|\mathbf z_{\lambda,h}(\mathbf x)-\mathbf z_{\lambda',h}(\mathbf x)\|_2
\le
L_{h}|\lambda-\lambda'|\,,
\qquad
\forall \mathbf x\in\mathcal X,\ \lambda \text{ and } \lambda' \text{ two corruption levels}.
\]
Assume also that the corresponding oracle slot score \(s^\circ_{.,h}\) is
\(M_{h}\)-Lipschitz in its feature argument over the attained range. Then
\[
|s^\circ_{\lambda,h}(\mathbf x)-s^\circ_{\lambda',h}(\mathbf x)|
\le
M_{h}L_{h}|\lambda-\lambda'|.
\]
Consequently, for a discrete backbone using a native matched level \(\lambda_t\) and a requested canonical level \(\lambda_{\mathrm{ref}}\), by averaging these bounds over the selected hooks in the oracle
 \cfs\ score :
\[
|S^\circ_{\cfs,\lambda_{\mathrm{ref}}}(\mathbf x)-S^\circ_{\cfs,\lambda_t}(\mathbf x)|
\le
\frac{1}{|\mathcal S|}
\sum_{h\in\mathcal S}
M_{h}L_{h}|\lambda_{\mathrm{ref}}-\lambda_t|\,.
\]

The same reasoning can be applied to AUROC score difference.



 This diagnostic is only meaningful for discrete backbones (it is detailed in Appendix~\ref{app:canon_mapping}), so we report it on improved-diffusion. Figure~\ref{fig:theory_canonical_appendix} shows two complementary views. First, larger effective canonical mismatches tend to induce larger OOD score drift. Second, sufficiently large mismatches can also produce measurable AUROC degradation relative to the matched logSNR-uniform reference policy. The relation is not perfectly linear, which is expected under discrete timestep collisions and pair-dependent difficulty, but the global trend supports the role of canonical logSNR matching as a genuine methodological requirement rather than a cosmetic alignment choice.

 \begin{figure*}[t]
\centering
\includegraphics[width=.48\textwidth]{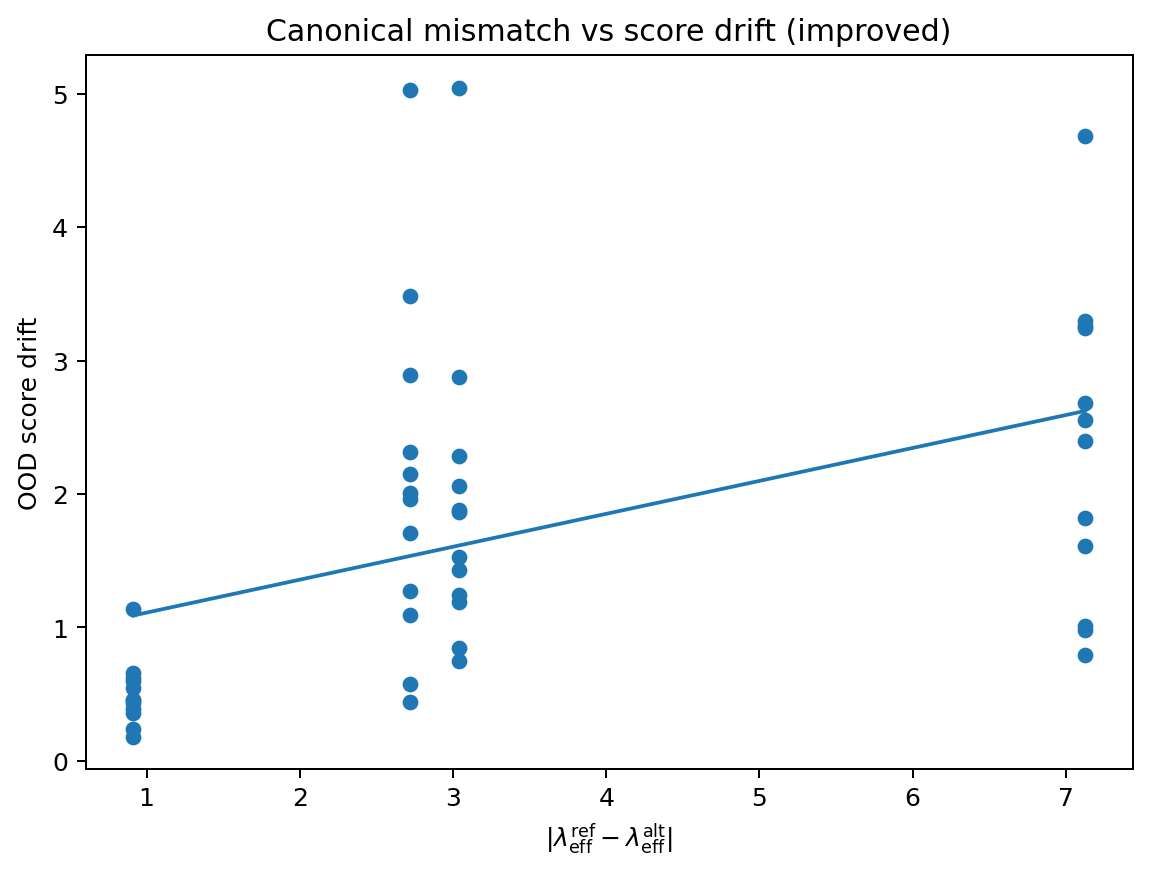}
\hfill
\includegraphics[width=.48\textwidth]{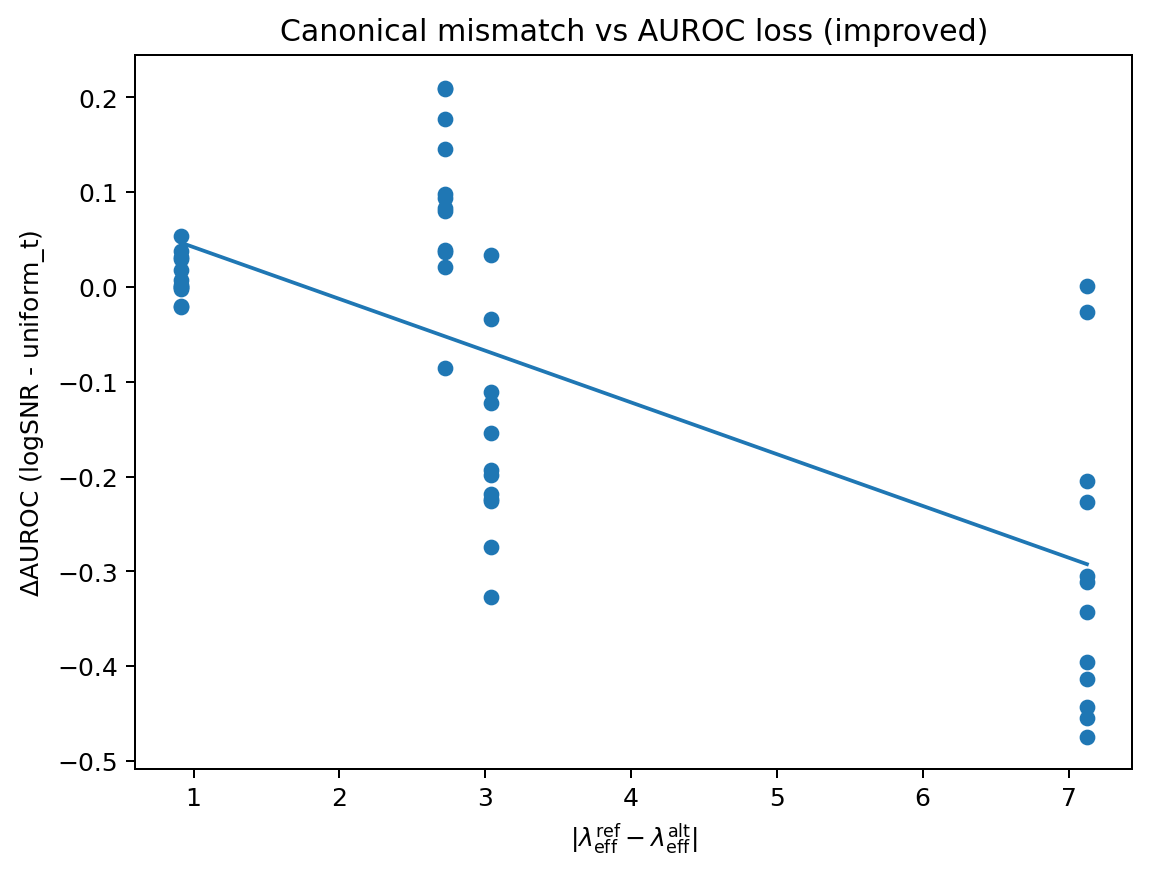}
\caption{
\textbf{Canonical-level matching sanity check on improved-diffusion.}
\textbf{Left:} standardized score drift versus effective logSNR mismatch
\(|\lambda_{\mathrm{ref}}-\lambda_t|\). 
\textbf{Right:} AUROC degradation relative to the matched logSNR policy.
Large mismatches can induce score drift and degrade AUROC, supporting logSNR matching as an
implementation requirement rather than a cosmetic alignment choice.
}
\label{fig:theory_canonical_appendix}
\end{figure*}

\subsection{Empirical diagnostic protocol}
\label{app:empirical_diagnostic_protocol}

The theory of Section~\ref{sec:method} yields two measurable quantities used in our empirical diagnostics: the diagonal noncentrality \(\hat\kappa_\lambda(\mathcal S)\) and the
content-to-instability ratio \(\hat R_h(\lambda)\). These quantities are estimated on held-out data and compared directly against downstream OOD performance.

For \(\hat\kappa_\lambda(\mathcal S)\), we study rank correlation with pairwise AUROC across candidate sparse probes and ID\(\to\)OOD pairs. This diagnostic is summarized in the main paper through Table~\ref{tab:theory_diag_summary}, with representative scatter plots deferred to Figure~\ref{fig:theory_kappa_appendix}. It confirms this prediction on both improved-diffusion and EDM backbones. Across candidate sparse probes and ID\(\to\)OOD pairs, larger values of \(\hat\kappa_\lambda(\mathcal S)/d\) are strongly aligned with larger AUROC, supporting the interpretation of the diagonal score as a structured local detector rather than as an arbitrary lightweight classifier.

For \(\hat R_h(\lambda)\), we study rank correlation with average AUROC across candidate hooks, and we additionally measure within-image corruption variance as a function of \(b(\lambda)^2\). These validations are reported in
Figure~\ref{fig:theory_ratio_appendix}.


\begin{figure*}[t]
\centering
\includegraphics[width=.48\textwidth]{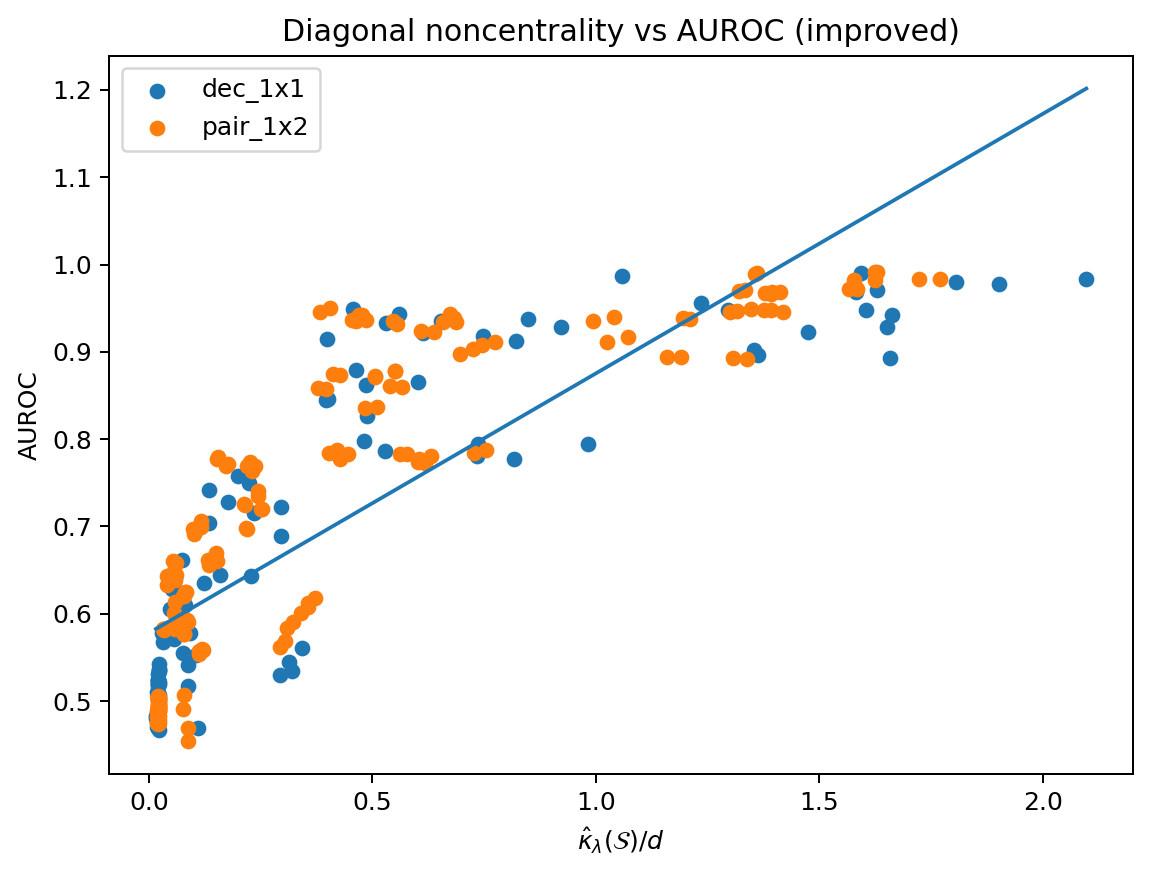}
\includegraphics[width=.48\textwidth]{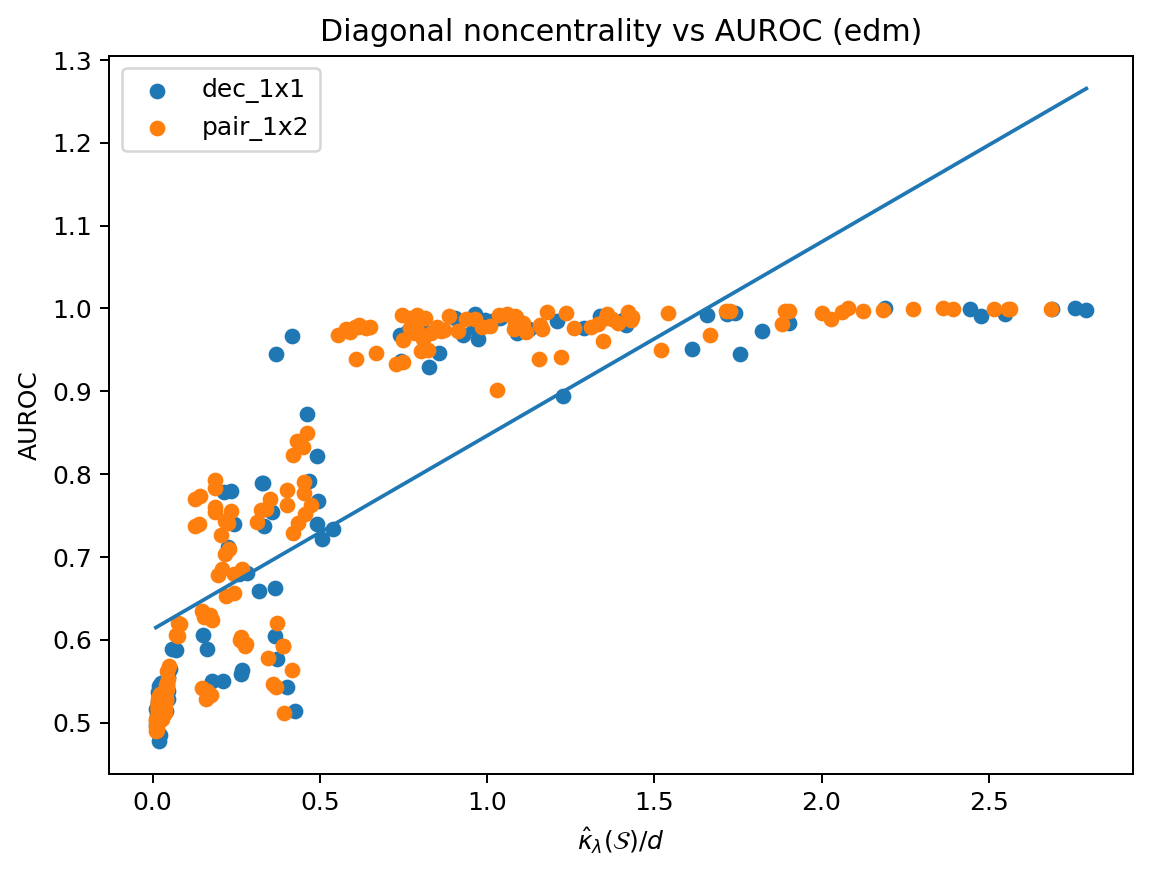}
\caption{
\textbf{Empirical validation of diagonal separation.}
The estimated diagonal noncentrality \(\hat\kappa_\lambda(\mathcal S)/d\) is strongly aligned with downstream AUROC on both improved-diffusion and EDM backbones.
}
\label{fig:theory_kappa_appendix}
\end{figure*}

\begin{figure*}[t]
\centering
\includegraphics[width=.48\textwidth]{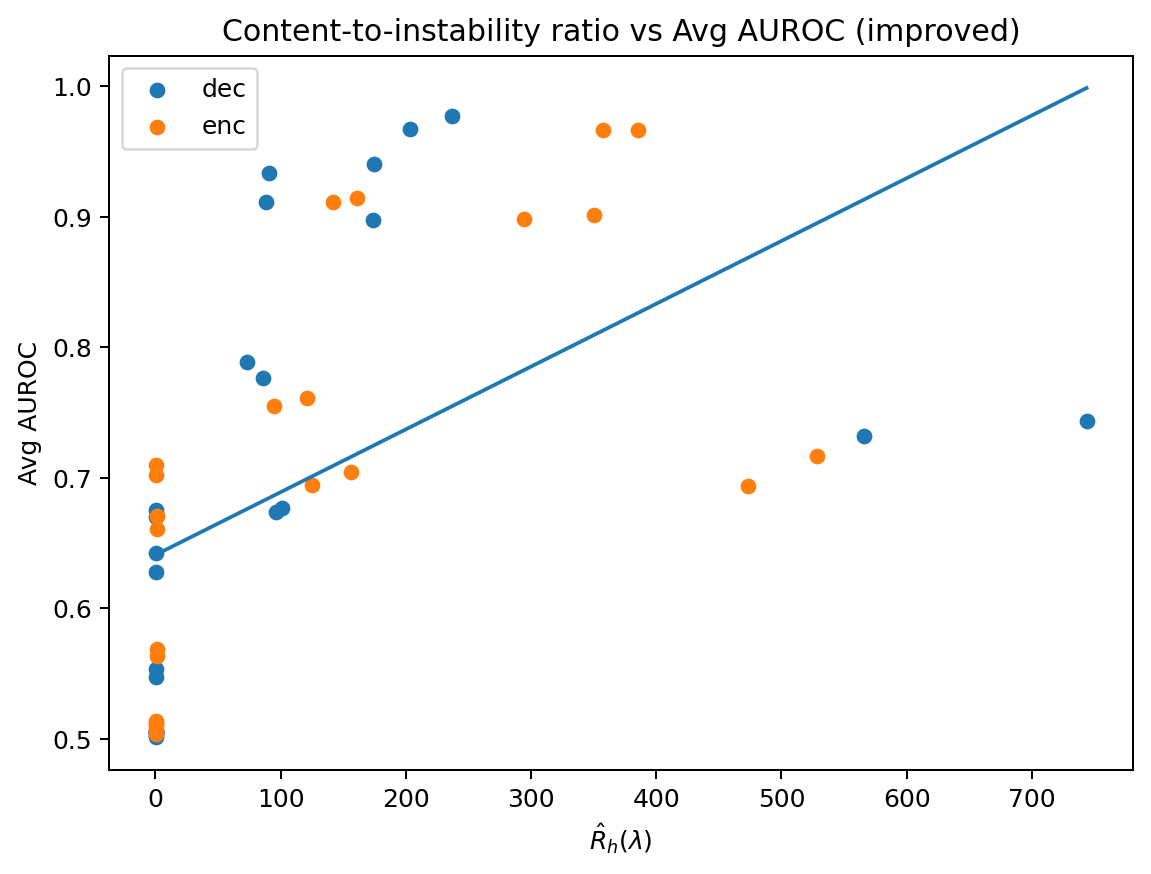}
\hfill
\includegraphics[width=.48\textwidth]{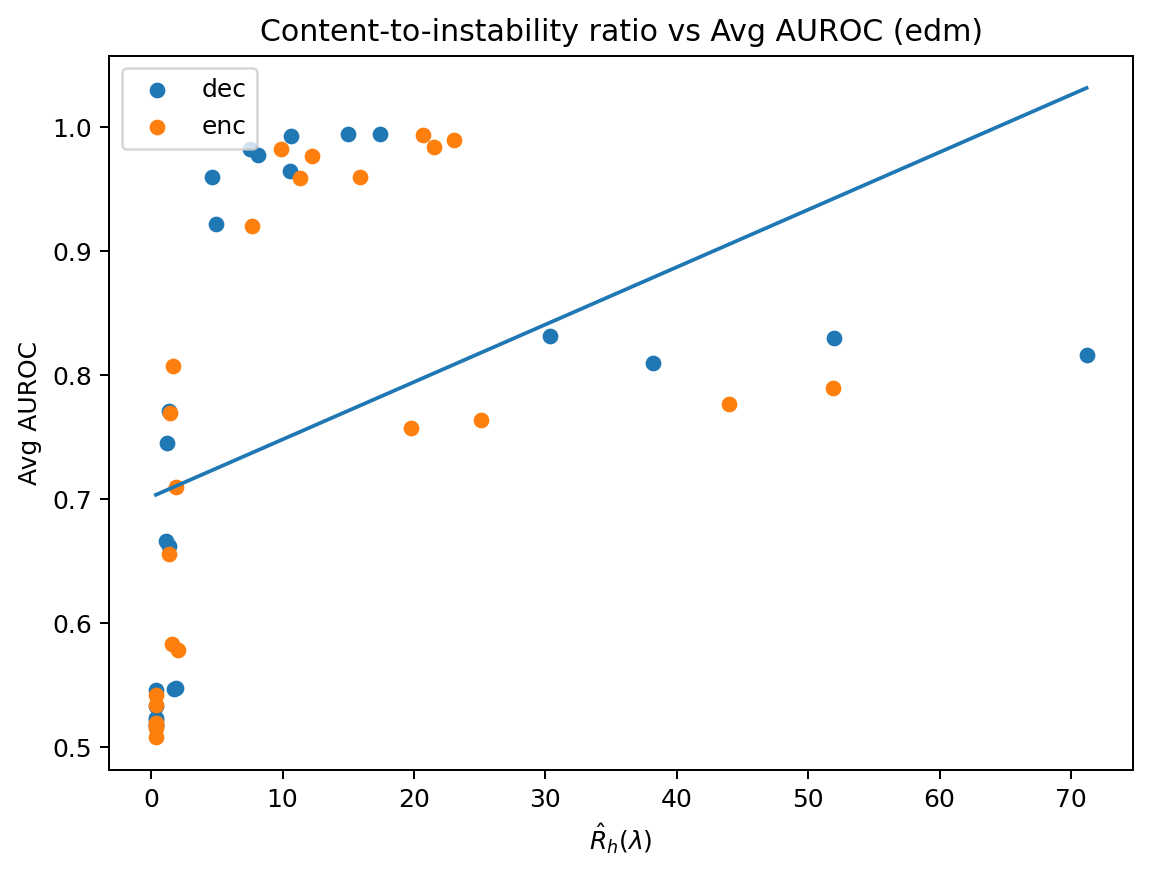}

\vspace{0.4em}

\includegraphics[width=.48\textwidth]{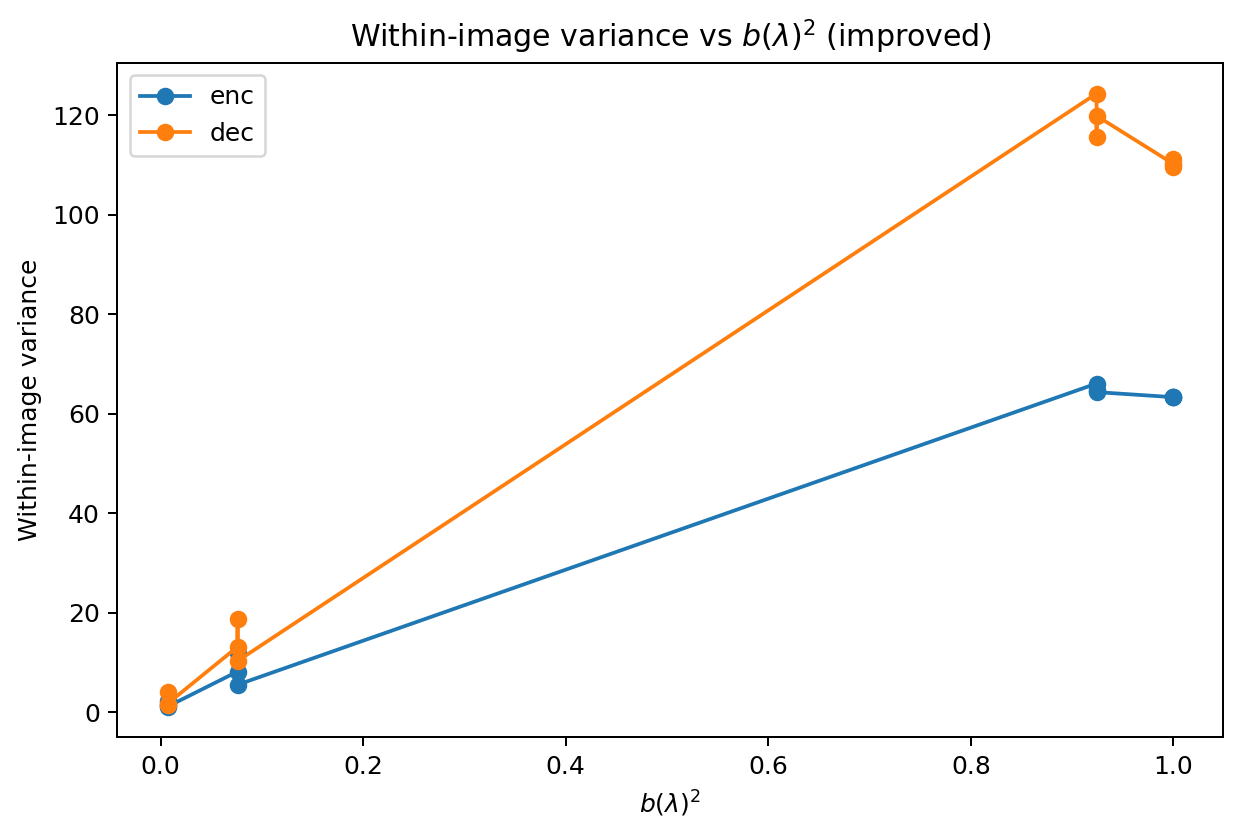}
\hfill
\includegraphics[width=.48\textwidth]{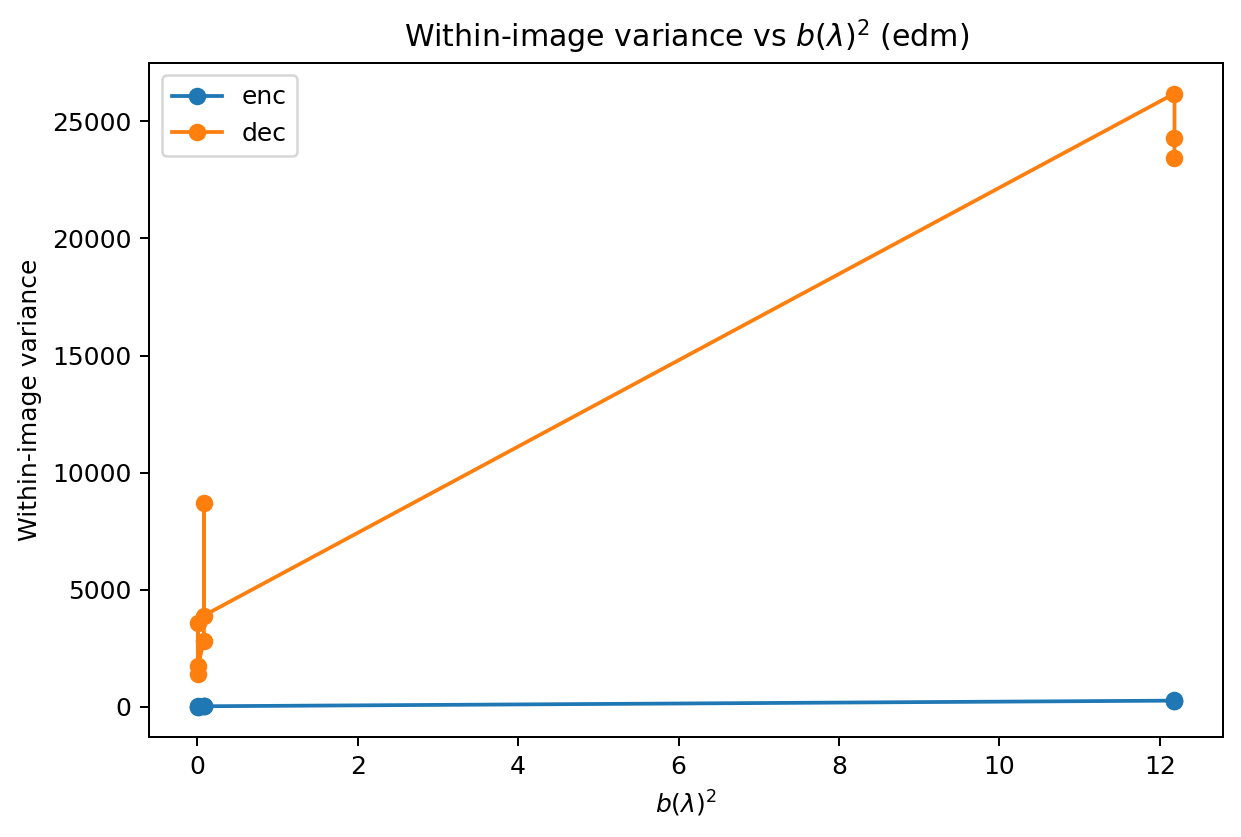}
\caption{
\textbf{Empirical validation of low-noise stability and the content-to-instability diagnostic.}
\textbf{Top row:} \(\hat R_h(\lambda)\) versus downstream Avg AUROC across candidate hooks for improved-diffusion (left) and EDM (right).
\textbf{Bottom row:} within-image corruption variance versus \(b(\lambda)^2\) for the selected encoder and decoder hooks on improved-diffusion (left) and EDM (right).
Across both backbone families, larger content-to-instability ratios align with stronger downstream performance, and lower-noise canonical probing reduces corruption-induced variability.
}
\label{fig:theory_ratio_appendix}
\end{figure*}

\subsection{Interpretation}
\label{app:theory_scope_interp}

Taken together, the appendix theory supports three concrete conclusions about sparse diffusion probing.

First, the paired encoder-decoder separation decomposes into a decoder term plus a nonnegative conditional encoder residual. This supports using a late decoder snapshot as the primary sparse probe, while interpreting encoder snapshots as complementary rather than primary. Second, the diagonal \cfs\ score is governed by a measurable diagonal noncentrality parameter \(\kappa_\lambda\), so the score can be analyzed as a structured local detector rather than as an arbitrary lightweight head. Third, low-noise canonical probing reduces corruption-induced instability: under a local smoothness approximation, within-image feature variance scales as \(b(\lambda)^2\), and the empirical content-to-instability diagnostic captures this effect.

When included, the canonical-matching diagnostic further supports that logSNR alignment is not merely a bookkeeping device: large discrete-level mismatches can induce score drift and degrade AUROC.

The theory remains local rather than universal. It does not claim that sparse internal probing dominates every conceivable output-space detector. Its purpose is narrower: to explain and predict which sparse internal probes should work best under a backbone-equated protocol.

\section{Experimental Protocols, Splits, and Canonicalization}
\label{app:protocols}

This section groups the protocol elements shared across methods: dataset splits, canonical level construction, backbone-specific corruption mappings, adapter outputs, and logical cost accounting.

\subsection{Small-scale benchmark protocol}
\label{app:small_protocol}

The small-scale benchmark uses the three ID datasets
\[
\mathcal{I}_{\text{small}}=\{\text{CIFAR-10},\text{SVHN},\text{CelebA32}\},
\]
and the OOD pool
\[
\mathcal{O}_{\text{small}}=\{\text{CIFAR-10},\text{SVHN},\text{CelebA32},\text{CIFAR-100},\text{DTD}\}.
\]
For each ID dataset, all remaining datasets in the pool are treated as OOD, producing \(12\) ID\(\to\)OOD pairs per backbone.

For each ID dataset, we define:
\begin{itemize}
\item an ID-fit split used to fit ID-only statistics and density heads;
\item an ID-test split used for in-distribution evaluation;
\item external OOD-test splits for all OOD datasets.
\end{itemize}
These splits are fixed and shared across methods.

\paragraph{Relative-OOD interpretation.}
Let \(P_\star\) denote the training distribution of the frozen checkpoint, and let \(P\) denote the chosen evaluation ID dataset. Our detector should be interpreted as performing OOD detection \emph{relative to the evaluation reference bank} \(P\), not as an absolute test of membership in \(P_\star\). This distinction is especially important in cross-source or transfer settings.









\subsection{ImageNet200 / ImageNet1K protocol}
\label{app:imagenet_protocol}

ImageNet200 and ImageNet1K are used as checkpoint-controlled large-scale benchmarks under a single official ImageNet-64 improved-diffusion backbone. Since the official training split is not used in our setup, we adopt deterministic disjoint val-only protocols. For ImageNet200, we use a class-stratified split of the validation set into:
\begin{itemize}
\item \textbf{ID-fit / bank split:} used to fit ID statistics;
\item \textbf{ID-test split:} the complementary subset used for in-distribution evaluation.
\end{itemize}
For ImageNet1K, we use the same val-only deterministic split policy to define ID-fit and ID-test subsets.

In both benchmarks, the OOD suite contains two near-OOD sets, \textbf{NINCO} and \textbf{SSB-hard}, together with one far-OOD texture-shift set, \textbf{Textures}.

\paragraph{Checkpoint-controlled protocol.}
All large-scale results in this appendix use the same official ImageNet-64 improved-diffusion checkpoint (\texttt{imagenet64\_uncond\_100M\_1500K.pt}). This removes cross-family variability and isolates method differences under a single shared-source backbone.



\subsection{Canonical level construction: \(K_{\text{grid}}\) vs.\ \(K_c\)}
\label{app:kgrid}

\paragraph{Purpose.}
Discrete backbones require mapping continuous canonical levels to discrete timesteps, which may produce duplicates. We separate:
\begin{itemize}
\item \(K_{\text{grid}}\): candidate grid resolution used to build a dense set of possible canonical levels;
\item \(K_c\): number of selected levels actually used by the method.
\end{itemize}

\paragraph{Construction.}
We construct a candidate grid
\[
\{\lambda_i^{\mathrm{grid}}\}_{i=1}^{K_{\text{grid}}}
\subset [\lambda_{\min},\lambda_{\max}]\,.
\]
For discrete backbones, each candidate level is mapped to the nearest available timestep in logSNR space. When \texttt{unique=true}, repeated timestep assignments are removed, and the final selected set contains at most \(K_c\) distinct levels. For continuous backbones, the selected levels are evenly spaced directly in canonical space.

\paragraph{Interpretation.}
Increasing \(K_{\text{grid}}\) improves mapping fidelity on discrete backbones without necessarily increasing test-time cost, because only the final selected \(K_c\) levels are used downstream.

\begin{remark}
The separation between \(K_{\text{grid}}\) and \(K_c\) is important for fair backbone matching: selection fidelity and evaluation cost should not be conflated.
\end{remark}

\subsection{Backbone-specific realization of canonical corruption}
\label{app:canon_mapping}

\subsubsection{Improved-diffusion (VP-DDPM)}

For improved-diffusion checkpoints, the native forward corruption process is
\[
\mathbf{x}_t=\sqrt{\bar\alpha_t}\,\mathbf{x}_0+\sqrt{1-\bar\alpha_t}\,\boldsymbol{\eps},
\qquad \boldsymbol{\eps}\sim\mathcal N(\mathbf{0}, \mathbf{I})\, ,
\]
where \(\bar\alpha_t\) is the cumulative noise-schedule coefficient of the checkpoint.

Hence, the corresponding corruption coefficients in the shared canonical form
\[
\mathbf{x}_\lambda = a(\lambda)\, \mathbf{x}_0 + b(\lambda)\, \boldsymbol{\eps}\, ,
\]
are
\[
a_t=\sqrt{\bar\alpha_t}\,,
\qquad
b_t=\sqrt{1-\bar\alpha_t}\, ,
\]
and the induced discrete canonical logSNR is
\[
\lambda_t=\log\frac{\bar\alpha_t}{1-\bar\alpha_t}\, .
\]

A desired canonical level \(\lambda\) is therefore mapped to the nearest native timestep in logSNR space:
\[
t(\lambda)=\arg\min_t |\lambda_t-\lambda|\,.
\]
When several candidate canonical levels map to the same native timestep, duplicates are removed by the unique-level construction described in Appendix~\ref{app:kgrid}.

\subsubsection{EDM family}

EDM-style models expose a continuous noise variable, but the exact native parameterization depends on the checkpoint preconditioning family. Rather than matching checkpoints through their native interface directly, we canonicalize them through the effective corruption ratio
\[
\tilde\sigma(\lambda):=\frac{b(\lambda)}{a(\lambda)}=\exp(-\lambda/2)\,,
\]
which follows from
\[
\lambda=\log\frac{a(\lambda)^2}{b(\lambda)^2}\,.
\]

The adapter, therefore, maps each canonical level \(\lambda\) to the checkpoint-specific model input corresponding to the same effective ratio \(\tilde\sigma(\lambda)\), and returns coefficients \((a,b)\) such that
\[
\mathbf{x}_\lambda = a \, \mathbf{x}_0 + b \, \boldsymbol{\eps}\,.
\]

Different EDM-family preconditionings may realize the same \(\tilde\sigma\) with different coefficient pairs \((a,b)\). For example, a VE/EDM-style realization uses
\[
a=1,\qquad b=\tilde\sigma\,,
\]
whereas a VP-style realization can be written as
\[
a=(1+\tilde\sigma^2)^{-1/2}\, ,
\qquad
b=\tilde\sigma(1+\tilde\sigma^2)^{-1/2}\,.
\]
In both cases,
\[
\frac{b}{a}=\tilde\sigma\,,
\]
so the same canonical \(\lambda\) corresponds to the same effective corruption strength even though the native checkpoint interface differs.

\subsubsection{Shared canonical semantics}

Although improved-diffusion and EDM differ in interface and preconditioning, all methods operate through the same canonical corruption abstraction. This allows meaningful cross-backbone comparisons at matched canonical levels.

\subsection{Shared adapter outputs and reconstruction rules}
\label{app:shared_x0_eps}

All harmonized methods are routed through a shared adapter layer. At each canonical level, the adapter provides:
\begin{itemize}
\item a denoised estimate \(\hat{\mathbf{x}}_0\),
\item optional access to intermediate activations,
\item and, when needed, an \(\hat{\boldsymbol{\eps}}\) estimate recovered through
\begin{equation}
\hat{\boldsymbol{\eps}} = \frac{\mathbf{x}-a\,\hat{\mathbf{x}}_0}{b}\,.
\label{eq:shared_epshat_app}
\end{equation}
\end{itemize}

This ensures that all output-space baselines and internal-feature methods use a compatible corruption semantics, regardless of the native output conventions of the original implementation.



\subsection{Implementation details specific to \cfs}
\label{app:cfs_impl_details}

The conceptual design of \cfs\ is described in Section~\ref{sec:method}. Here we report only implementation-level details needed for exact reproduction.

\subsubsection{Admissible hook search and shortlist size}

A candidate hook is admissible only if a dry forward pass returns a 4D tensor of shape \(B\times C\times H\times W\). In practice, we restrict the search to stage-level encoder and decoder blocks and keep a small shortlist within each structural region before applying the ID-only proxy. For every experiment, we report:
\begin{itemize}
\item the number of admissible encoder candidates,
\item the number of admissible decoder candidates,
\item the shortlist size retained in each region,
\item the final selected encoder and decoder module names.
\end{itemize}

\subsubsection{ID-only probe configuration}

The proxy uses a small ID-only probe set. For exact reproducibility, we report:
\begin{itemize}
\item the number of ID probe images,
\item the number of corruption repeats per image,
\item whether the same noise draw is reused across canonical levels,
\end{itemize}

\subsubsection{Exact pooled-slot construction}

For each retained slot, we apply channel-wise spatial mean and standard deviation pooling, yielding a descriptor of dimension \(2C_{k,\ell}\). We report, for each backbone family and benchmark:
\begin{itemize}
\item the selected canonical levels,
\item the selected encoder and decoder block names,
\item the resulting feature-map shapes,
\item the pooled descriptor dimensions.
\end{itemize}

\subsubsection{Diagonal score details}

Each pooled slot is modeled independently with ID-only diagonal statistics. For exact reproduction, we report:
\begin{itemize}
\item whether the slot features are standardized before fitting,
\item whether slot scores are averaged uniformly or reweighted.
\end{itemize}

\section{Implementation Taxonomy and Baseline Specifications}
\label{app:taxonomy_baselines}

\subsection{Implementation taxonomy}
\label{app:taxonomy}

We implement these methods to preserve the core probe and score logic of the original method while routing it through our shared adapter and canonicalization pipeline.



\paragraph{Methods in this paper.}
\begin{itemize}
\item \msma: harmonized faithful port;
\item \diffpath: harmonized faithful port;
\item \ddpmood: harmonized faithful port;
\item \gepc: harmonized faithful port;
\item \cfs: proposed internal representation-space method.
\end{itemize}

\paragraph{Methods discussed but not retained in the strict main benchmark.}
We position SCOPED and EigenScore conceptually in related work, but do not include them in the strict main \mbe\ table. In our attempts, we did not obtain a controlled, harmonized rerun of SCOPED suitable for the shared benchmark. EigenScore was also not retained because matched reruns within the shared adapter/canonicalization pipeline were substantially heavier, and some public checkpoint/artifact combinations did not yield reliable runs. We therefore restrict the strict main benchmark to methods for which we can provide controlled backbone-equated evaluation under matched corruption semantics and budget accounting.

\paragraph{Official repositories used as starting points.}
When available, our harmonized implementations were initialized from the official public repositories of the corresponding methods:
MSMA\footnote{\url{https://github.com/ahsanMah/msma}},
DiffPath\footnote{\url{https://github.com/clear-nus/diffpath}},
DDPM-OOD\footnote{\url{https://github.com/marksgraham/ddpm-ood}},
and GEPC\footnote{\url{https://github.com/RouzAY/gepc-diffusion}}.
These repositories were then adapted to the shared adapter, canonicalization, and logical-budget interface of \mbe.

\subsection{Common harmonized interface}
\label{app:baselines_common}

All baselines are evaluated through the same adapter interface, the same input normalization to \([-1,1]\), and the same canonical corruption semantics. For a selected canonical level \(\lambda_k\) with coefficients \((a_k,b_k)\), we explicitly corrupt the clean image as
\[
\mathbf{x}_k = a_k \, \mathbf{x}_0 + b_k \, \boldsymbol{\eps}, \qquad \boldsymbol{\eps} \sim \mathcal{N}(\mathbf{0},\mathbf{I})\,.
\]
From the frozen backbone, we then recover native denoising quantities $\hat{\mathbf{x}}_{0,k}$ and $\hat{\boldsymbol{\eps}}_k$.
For improved-diffusion backbones, \(\hat\eps_k\) is extracted natively from the model output using the correct timestep scaling and mean-parameterization conventions, and \(\hat{\mathbf{x}}_{0,k}\) is recovered from the same forward pass. For EDM-style backbones, the adapter returns \(\hat{\mathbf{x}}_{0,k}\) through one denoising call and reconstructs
\[
\hat{\boldsymbol{\eps}}_k = \frac{\mathbf{x}_k - a_k \, \hat{\mathbf{x}}_{0,k}}{b_k}\, .
\]
Thus, one backbone evaluation at one level counts as one logical forward pass, even if both \(\hat{\mathbf{x}}_0\) and \(\hat{\boldsymbol{\eps}}\) are recovered from that same call.


\subsection{MSMA (harmonized faithful port)}
\label{app:msma_details}

MSMA is the closest baseline in our suite to a harmonized faithful port. Its core logic is preserved: build a multiscale descriptor from denoiser-derived quantities across \(K_c\) levels, then fit an ID-only density head on the resulting feature vectors.

\paragraph{Feature construction.}
For each canonical level \(\lambda_k\), we compute
\[
f_k(\mathbf{x}_0)=\|\hat{\boldsymbol{\eps}}_k\|_2\, ,
\]
where \(\hat{\boldsymbol{\eps}}_k\) is obtained through the shared native adapter interface. The final multiscale descriptor is
\[
F(\mathbf{x}_0)=\big[f_1(\mathbf{x}_0),\dots,f_{K_c}(\mathbf{x}_0)\big]\in\mathbb{R}^{K_c}\, .
\]

\paragraph{Head and scoring.}
We fit an ID-only density model on \(F(x)\), after optional feature standardization. Our implementation supports:
\begin{itemize}
\item a diagonal Gaussian head,
\item a Gaussian mixture model (GMM),
\item a \(k\)-nearest-neighbor distance head.
\end{itemize}
The final score is always OOD-high: negative log-likelihood for Gaussian/GMM heads, or distance for the KNN head.

\paragraph{What is preserved and what is adapted.}
The preserved part is the multiscale descriptor and the density score. The adapted part is the use of a common adapter interface, canonical logSNR levels, and a shared improved/EDM implementation.

\subsection{DiffPath (harmonized faithful port)}
\label{app:diffpath_details}

DiffPath is implemented as a faithful port for path-based diffusion OOD detection. Rather than reproducing a native implementation verbatim, we preserve the key idea: summarize a \emph{multilevel denoising path} and fit an ID-only density model on the resulting path statistics.

\paragraph{Recursive path construction.}
Levels are ordered from clean to noisy. We initialize
\[
\mathbf{x}_{1} = a_1 \, \mathbf{x}_0 + b_1 \, \boldsymbol{\eps}\, ,
\]
and for each level \(k\) compute \((\hat{\mathbf{x}}_{0,k},\hat{\boldsymbol{\eps}}_k)\) from the current state \({\mathbf{x}}_k\). A scalar path statistic \(q_k\) is then extracted from \(\hat{\boldsymbol{\eps}}_k\) using one of the following reductions:
\[
q_k \in \Big\{
\mathrm{mean}(\hat{\boldsymbol{\eps}}_k),\,
\mathrm{mean}(|\hat{\boldsymbol{\eps}}_k|),\,
\mathrm{mean}(\hat{\boldsymbol{\eps}}_k^2)
\Big\}\,.
\]
We then propagate the path recursively:
\[
\mathbf{x}_{k+1} = a_{k+1}\, \hat{\mathbf{x}}_{0,k} + b_{k+1}\, \hat{\boldsymbol{\eps}}_k\,.
\]

\paragraph{Feature variants.}
We use two harmonized feature families. The 1D variant computes
\[
\phi_{\mathrm{1d}}(x)
=
\sqrt{
\frac{1}{K_c-1}
\sum_{k=1}^{K_c-1}
\left(
\frac{q_{k+1}-q_k}{\Delta\lambda_k}
\right)^2
}\,,
\]


The 6D variant computes low-order moments of both the path values \(Q=(q_1,\ldots,q_{K_c})\) and their level-wise differences
\(\Delta Q=(q_2-q_1,\ldots,q_{K_c}-q_{K_c-1})\):
\[
\phi_{\mathrm{6d}}(x)
=
\big[
\operatorname{mean}(Q),\,
\operatorname{mean}(Q^2),\,
\|Q\|_3,\,
\operatorname{mean}(|\Delta Q|),\,
\operatorname{mean}((\Delta Q)^2),\,
\|\Delta Q\|_3
\big].
\]

\paragraph{Head and scoring.}
We fit either a 1D KDE for the 1D feature or a diagonal Gaussian model for the 6D feature. Scores are OOD-high via negative log-density.

\subsection{DDPM-OOD (harmonized faithful port)}
\label{app:ddpmood_details}

Our DDPM-OOD baseline is implemented as a harmonized \emph{multi-start reconstruction}. It preserves the main scoring logic of DDPM-OOD: reconstruction error should be evaluated from several noisy starting points, normalized using ID statistics \emph{per start}, and then aggregated into a single OOD score.

\paragraph{Canonical reverse schedule and starting points.}
We first build a canonical level set
\[
\lambda_1 > \lambda_2 > \cdots > \lambda_{K_c}\,,
\]
ordered from clean to noisy, and reverse it into a noisy-to-clean reconstruction schedule. A subset of starting points is then selected by subsampling this reverse schedule.

\paragraph{Multi-start reconstruction.}
For a selected starting level \(\lambda^{(s)}\), we generate
\[
\mathbf{x}_{s} = a_{s} \, \mathbf{x}_0 + b_{s} \, \boldsymbol{\eps}\, ,
\qquad \boldsymbol{\eps} \sim \mathcal{N}(\mathbf{0}, \mathbf{I})\,,
\]
and reconstruct deterministically along the corresponding reverse suffix. If the suffix levels are denoted
\[
\lambda^{(s)}=\lambda_{j_1},\lambda_{j_2},\dots,\lambda_{j_m}\,,
\]
then at each step, we compute \((\hat x_{0,j_r},\hat\eps_{j_r})\) from the current state and propagate toward the next cleaner canonical level:
\[
\mathbf{x}_{j_{r+1}} = a_{j_{r+1}} \, \hat{\mathbf{x}}_{0,j_r} + b_{j_{r+1}} \, \hat{\boldsymbol{\eps}}_{j_r}\,.
\]
The final clean reconstruction is the terminal \(\hat{\mathbf{x}}_0\) at the end of the suffix.

\paragraph{Per-start reconstruction errors and ID normalization.}
For each start \(s\), we compute the reconstruction error
\[
m_s(\mathbf{x}_0)=\frac{1}{d}\, \left\|\hat{\mathbf{x}}_{0,s}-\mathbf{x}_0\right\|_2^2\,.
\]
This yields a vector of per-start errors
\[
M(\mathbf{x}_0)=\big[m_1(\mathbf{x}_0),\dots,m_S(\mathbf{x}_0)\big]\,.
\]
For each start \(s\), we fit ID-only normalization statistics and form
\[
z_s(\mathbf{x})=\frac{m_s(\mathbf{x})-c_s}{\sigma_s}\,.
\]

\paragraph{Aggregation.}
The final scalar score is obtained by aggregating the per-start normalized deviations:
\[
S(\mathbf{x})=\operatorname{Agg}\big(z_1(\mathbf{x}),\dots,z_S(\mathbf{x})\big)\, ,
\]
where \(\operatorname{Agg}\) denotes mean, median, or sum.

\subsection{GEPC (harmonized faithful port)}
\label{app:gepc_details}

Our GEPC baseline is implemented as a MBE-adapted output-consistency faithful port. The central idea is preserved: if the denoiser behaves approximately equivariantly under a small discrete transformation group, then in-distribution inputs should exhibit stronger posterior consistency than OOD inputs.

\paragraph{Canonical corruption and transformed outputs.}
For each selected canonical level \(\lambda_k\), we first form
\[
\mathbf{x}_k = a_k \, \mathbf{x}_0 + b_k \, \boldsymbol{\eps}\, ,
\qquad \boldsymbol{\eps} \sim \mathcal{N}(\mathbf{0}, \mathbf{I})\, ,
\]
then evaluate the denoiser on both \(\mathbf{x}_k\) and transformed copies \(g\cdot \mathbf{x}_k\), where \(g\) belongs to a small discrete group such as flips or \(180^\circ\) rotation. From each output we recover \(\hat{\mathbf{x}}_{0,k}\) and \(\hat{\boldsymbol{\eps}}_k\), and build the score proxy
\[
\hat{\mathbf{s}}_k = -\frac{\hat{\boldsymbol{\eps}}_k}{b_k}\,.
\]

\paragraph{Consistency features.}
We compare the reference output to the transformed-and-brought-back outputs \(g^{-1}\cdot\hat{\mathbf{s}}_k(g\cdot \mathbf{x}_k)\) and \(g^{-1}\cdot\hat {\mathbf{x}}_{0,k}(g\cdot \mathbf{x}_k)\). This yields several OOD-high consistency features, including a normalized score discrepancy, a cosine consistency score, and a normalized \(\hat{\mathbf{x}}_0\)-consistency score.

\paragraph{Level-wise calibration and aggregation.}
At each level, these raw consistency features are calibrated with an ID-only head, typically KDE or z-score normalization. Feature scores are then aggregated within each level, and finally across canonical levels using mean, weighted mean, or trimmed mean aggregation.

\section{Focused Ablations}
\label{app:ablations}

We restrict appendix ablations to controls that directly test the paper's central claim: the gain of \cfs\ comes from \emph{where} the frozen diffusion backbone is probed, rather than from hidden budget, brittle hook choices, or downstream head complexity.

\subsection{Pareto budget analysis and budget-matched comparisons}
\label{app:ablation_budget}

A standard concern in diffusion OOD comparisons is hidden compute. We therefore report a budget-aware comparison in which logical test-time cost is made explicit. The goal is simple: if output-space baselines are given a matched or larger budget, does the representation-space advantage persist?

More precisely, we vary the number of selected canonical levels while keeping Monte Carlo test \(\mathrm{MC}_{\mathrm{test}}=1\), so that the dominant logical cost scales with the number of backbone evaluations. For \cfs, \(K_c\times K_s\) denotes the number of selected canonical levels and retained stage slots. Since all retained hooks at a given level are extracted in the same forward pass, logical cost depends on \(K_c\), not on \(K_s\).

\begin{table*}[t]
\centering
\caption{
\textbf{Budget-aware comparison under \mbe.}
We vary the logical test-time budget while keeping \(\mathrm{MC}_{\mathrm{test}}=1\).
This disentangles representation quality from hidden multilevel accumulation.
For \cfs, \(K_c \times K_s\) denotes the number of selected canonical levels and retained stage slots, respectively.
Because all retained hooks at a given canonical level are extracted within the same backbone forward pass, the logical cost depends on \(K_c\) only, not on \(K_s\).
}
\label{tab:budget_nfe}
\small
\setlength{\tabcolsep}{5pt}
\resizebox{\textwidth}{!}{%
\begin{tabular}{cc|cc|cc}
\toprule
Budget & Method
& \multicolumn{2}{c|}{Improved-diffusion}
& \multicolumn{2}{c}{EDM} \\
(\#F/img) &
& \(\AvgAUROC \uparrow\) & \(\AvgWorstAUROC \uparrow\)
& \(\AvgAUROC \uparrow\) & \(\AvgWorstAUROC \uparrow\) \\
\midrule
1 & \cfssubvartag{enc}{1}{1}{low}  & 0.849 & 0.773 & 0.900 & 0.805 \\
1 & \cfssubvartag{dec}{1}{1}{low}  & \second{0.886} & \second{0.793} & \best{0.919} & \second{0.809} \\
1 & \cfsvartag{1}{2}{low}          & \best{0.887} & \best{0.799} & \second{0.916} & \best{0.814} \\
1 & \cfsvartag{1}{2}{high}         & 0.497 & 0.492 & 0.498 & 0.495 \\
1 & \cfssubvartag{enc}{1}{1}{high} & 0.500 & 0.496 & 0.498 & 0.494 \\
1 & \cfssubvartag{dec}{1}{1}{high} & 0.495 & 0.486 & 0.498 & 0.495 \\
\addlinespace
2 & \cfsvartag{2}{1}{enc}          & 0.839 & 0.761 & 0.880 & 0.781 \\
2 & \cfsvartag{2}{1}{dec}          & 0.853 & 0.762 & 0.906 & 0.800 \\
2 & \cfsvar{2}{2}                  & 0.868 & 0.779 & 0.902 & 0.804 \\
2 & \cfsvar{2}{4}                  & 0.873 & 0.783 & 0.900 & 0.794 \\
\addlinespace
2 & \msma                           & 0.781 & 0.652 & 0.770 & 0.645 \\
2 & \diffpath                       & 0.802 & 0.633 & 0.778 & 0.630 \\
\addlinespace
4 & \cfs\,\(4\times2\)                & 0.843 & 0.763 & 0.873 & 0.791 \\
4 & \msma                           & 0.785 & 0.669 & 0.790 & 0.673 \\
4 & \diffpath                       & 0.766 & 0.642 & 0.787 & 0.631 \\
\addlinespace
8 & \msma                           & 0.779 & 0.675 & 0.798 & 0.682 \\
8 & \diffpath                       & 0.779 & 0.640 & 0.792 & 0.634 \\
\bottomrule
\end{tabular}}
\end{table*}


Table~\ref{tab:budget_nfe} shows that the best one-forward operating points are already sparse \cfs\ variants. In particular, \(\cfs_{\mathrm{dec}}(1\times1_{\mathrm{low}})\) and \(\cfs(1\times2_{\mathrm{low}})\) dominate budget-matched output-space baselines while using only one backbone evaluation per image. Richer \cfs\ variants improve representation size without changing the backbone cost at fixed \(K_c\), but the gains beyond the strongest one-forward operating points are modest. The key conclusion is therefore not merely that \cfs\ is accurate, but that its gain does not come from hidden multilevel accumulation.

\subsection{Hook-pair robustness and proxy validation}
\label{app:ablation_hooks}

A natural concern is that the encoder-decoder pair selected by \cfs{}$(1\times2)$ could be fragile or cherry-picked. To test this, we evaluate multiple admissible encoder and decoder candidates within the same structural regions and report (i) the pairwise performance landscape, (ii) the relation between the ID-side pair proxy and final pair performance, and (iii) a complementary region-wise proxy validation. All results are computed at the fixed low-noise canonical level used in the main paper.

\begin{figure*}[t]
\centering


\includegraphics[width=.48\textwidth]{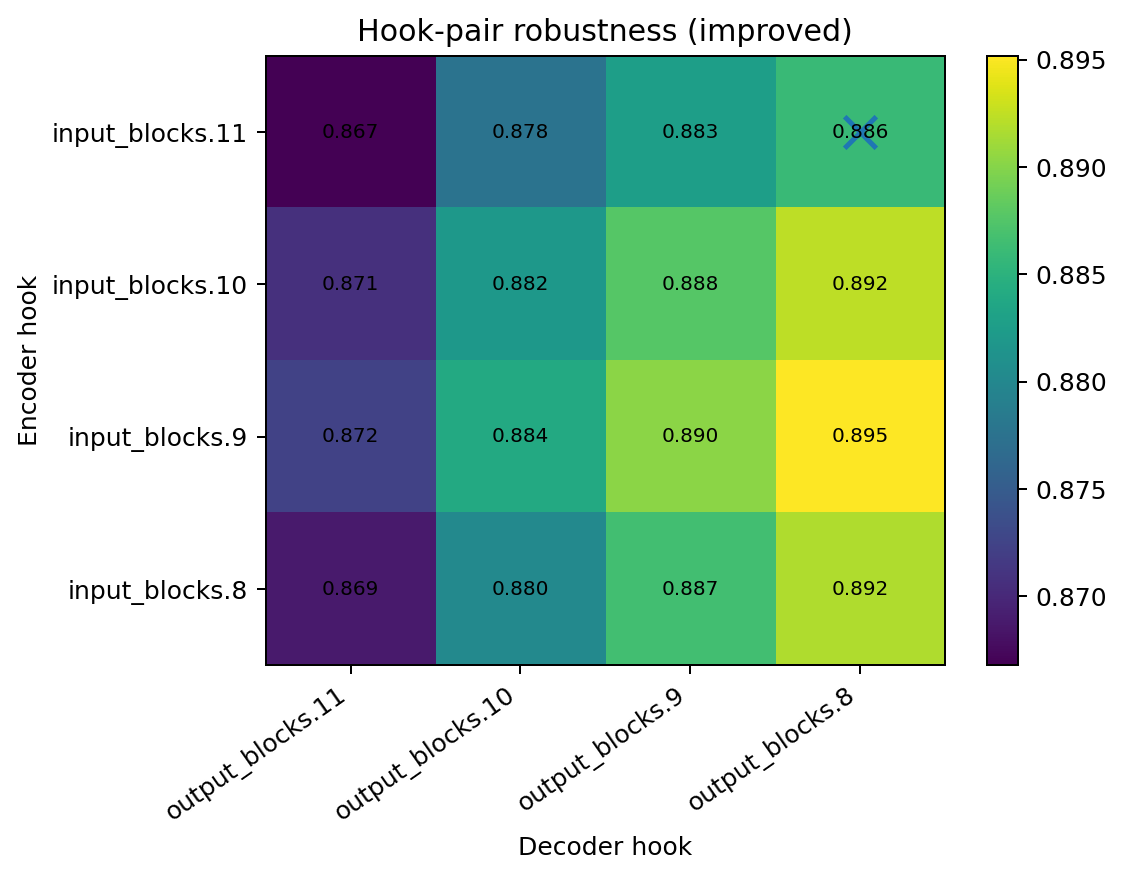}
\hfill
\includegraphics[width=.48\textwidth]{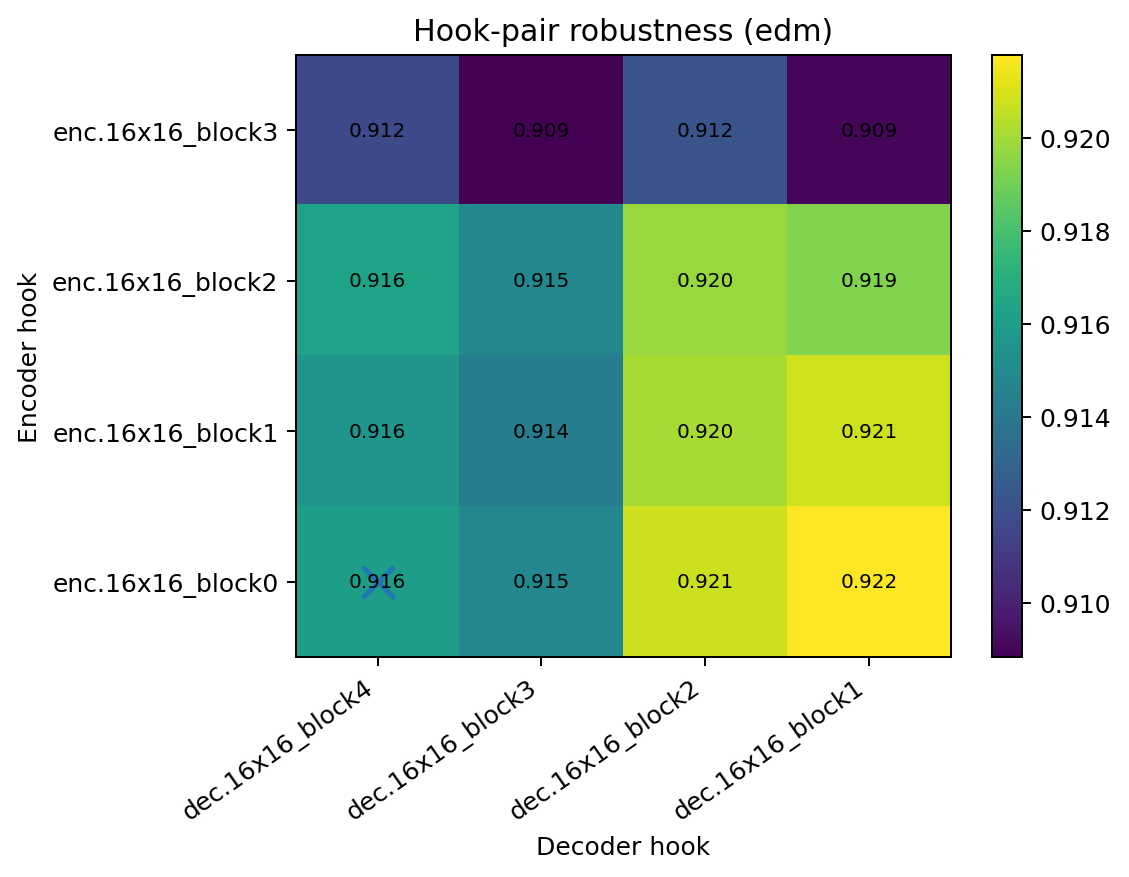}
\caption{
\textbf{Corresponding hook-pair heatmaps for \cfs{}$(1\times2)$.}
Each cell reports the Avg AUROC of one admissible pair at the low-noise canonical level used in the main paper.
The marker denotes the final pair obtained from the ID-side proxy shortlist.
For the improved backbone, the selected pair falls in the same high-performing basin as the oracle pair.
For EDM, the dominant pattern is a broad plateau of strong pairs, indicating low sensitivity to the exact hook choice even when the proxy is less predictive.
}
\label{fig:hook_pair_robustness}
\end{figure*}

Figure~\ref{fig:hook_pair_robustness} supports robustness more strongly than exact pair recovery. For the improved backbone, the landscape is structured rather than spiky: the proxy-selected pair reaches \(0.886\) Avg AUROC, versus \(0.895\) for the oracle admissible pair, so the proxy lands in the correct high-performing basin even without identifying the exact optimum. For EDM, the heatmap is flatter, and the main conclusion is different: several nearby pairs perform similarly well, so the method is not driven by a single brittle choice.



\begin{table}[t]
\centering
\caption{
\textbf{Region-wise proxy-selected modules versus empirical oracle modules for \cfs{}$(1\times2)$.}
For each backbone and region, we compare the downstream \(\AvgAUROC\) obtained with the proxy-selected module to that obtained with the best admissible module in hindsight, while keeping the opposite region fixed to its proxy-selected choice.
Small gaps indicate that the proxy-selected module lies close to the empirical oracle without using any OOD labels.
}
\label{tab:proxy_oracle_gap}
\small
\begin{tabular}{llccc}
\toprule
Backbone & Region & Proxy-selected \(\AvgAUROC \uparrow\) & Oracle \(\AvgAUROC \uparrow\) & Gap \(\downarrow\) \\
\midrule
Improved-diffusion & Encoder & 0.8861 & 0.8952 & 0.0091 \\
Improved-diffusion & Decoder & 0.8861 & 0.8861 & 0.0000 \\
EDM                & Encoder & 0.9161 & 0.9163 & 0.0002 \\
EDM                & Decoder & 0.9161 & 0.9218 & 0.0057 \\
\bottomrule
\end{tabular}
\end{table}

Table~\ref{tab:proxy_oracle_gap} shows that, on improved-diffusion, the proxy is more informative on the decoder side than on the encoder side, matching the stronger decoder-column structure in the heatmap. On EDM, the proxy is weaker as a regional ranker, but the oracle gaps remain small. Overall, the selection rule does not need to find the exact best pair; it only needs to place \(\cfs(1{\times}2)\) in a stable, high-performing part of the network without exhaustive search.

\subsection{Canonical-level robustness}
\label{app:ablation_level_heatmap}

The main paper identifies low-noise canonical probing as the strongest operating regime. We now test whether this effect reflects a stable region or a narrowly tuned choice of canonical level.

Since the final method uses single-level variants, we sweep the canonical level \(\lambda\) and evaluate two representative detectors: \(\cfs_{\mathrm{dec}}(1{\times}1)\) and \(\cfs(1{\times}2)\). The key question is whether performance remains strong over a neighborhood of low-noise levels, rather than peaking at a single hand-picked value.

\begin{figure*}[t]
\centering
\includegraphics[width=.48\textwidth]{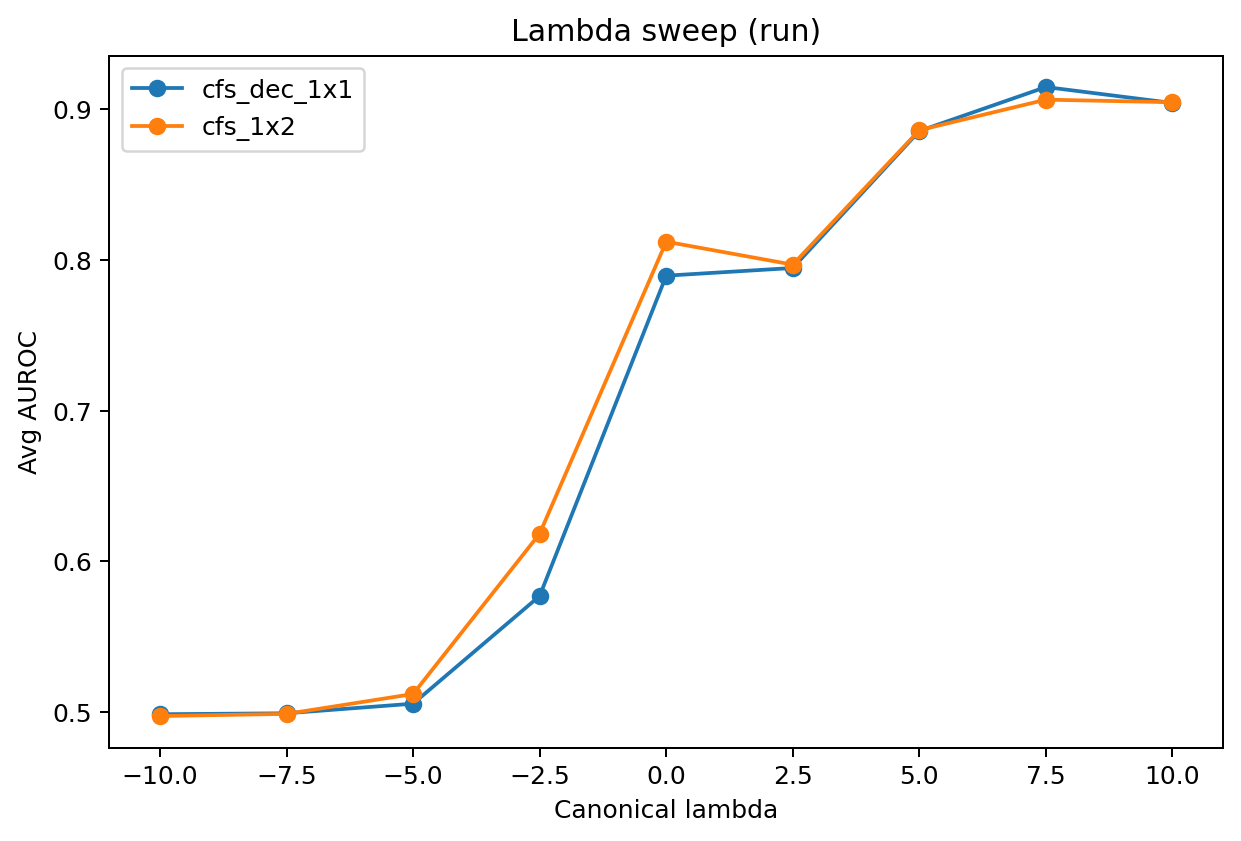}
\hfill
\includegraphics[width=.48\textwidth]{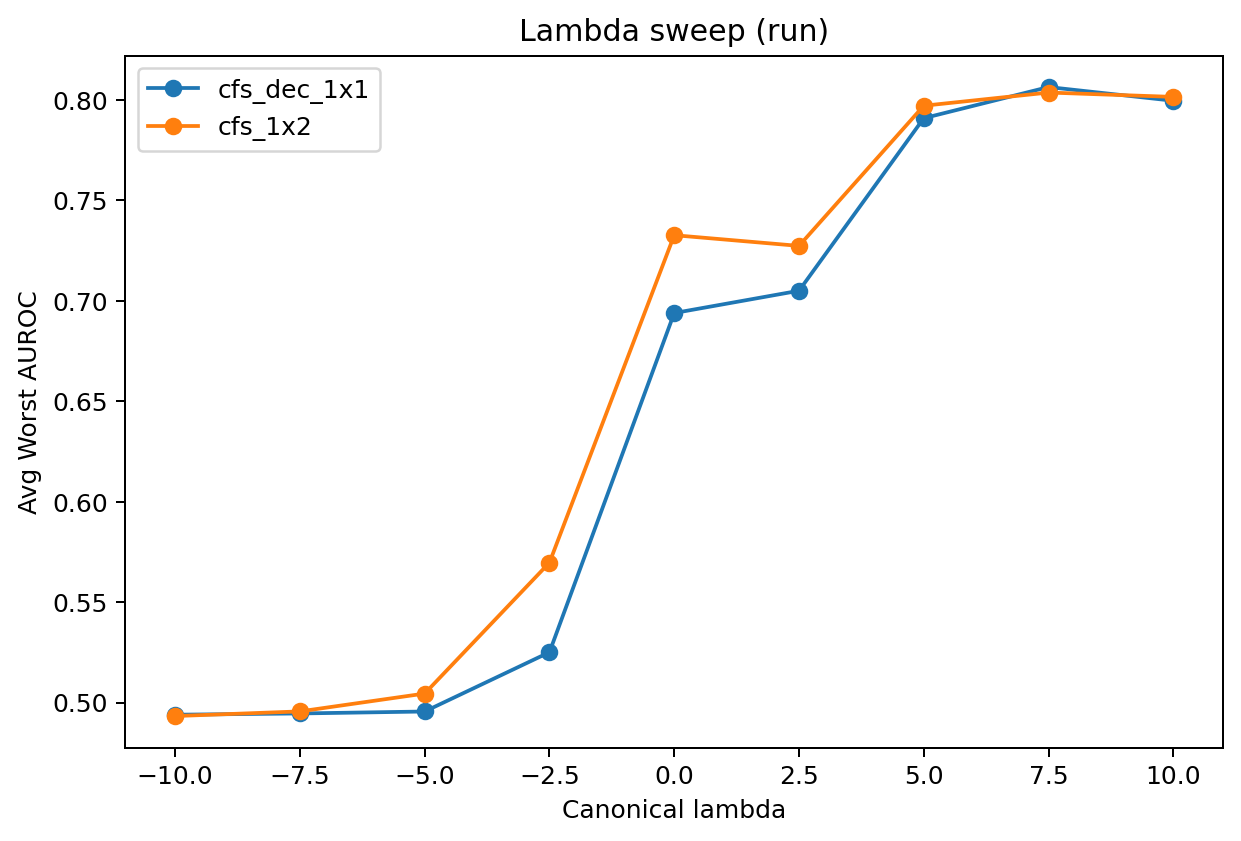}

\vspace{0.4em}

\includegraphics[width=.48\textwidth]{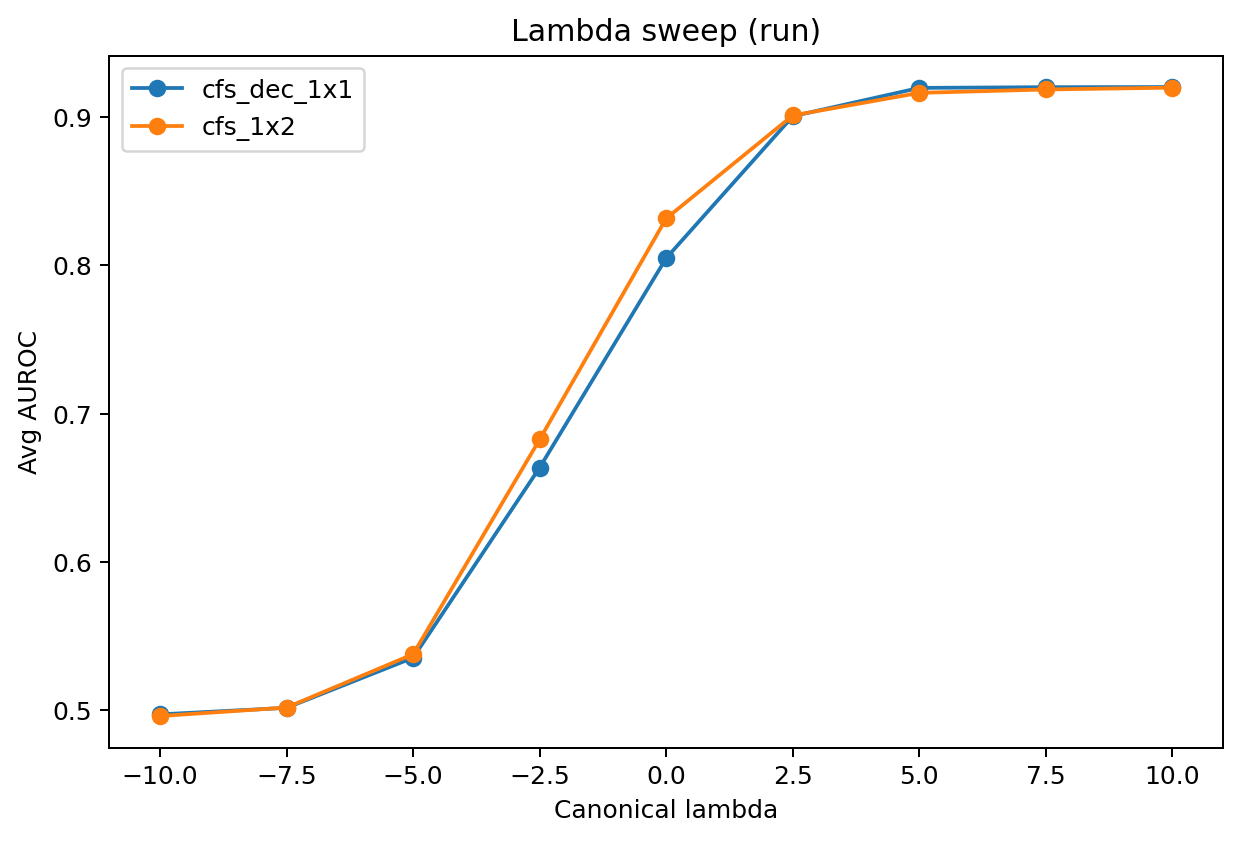}
\hfill
\includegraphics[width=.48\textwidth]{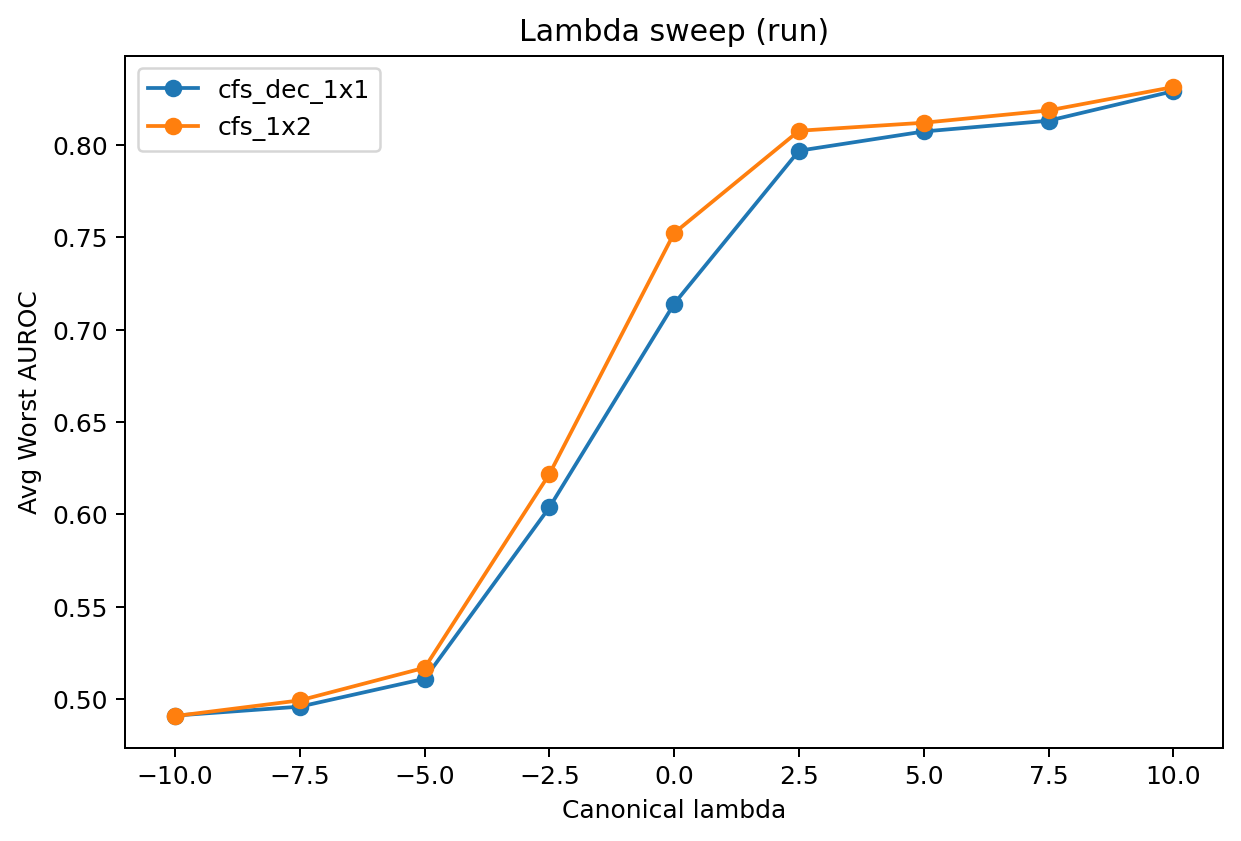}
\caption{
\textbf{Canonical-level robustness for single-level \cfs\ variants.}
\textbf{Top:} improved-diffusion backbone.
\textbf{Bottom:} EDM backbone.
\textbf{Left:} \(\AvgAUROC\). \textbf{Right:} \(\AvgWorstAUROC\).
Across both backbones, performance is near chance for strongly negative \(\lambda\), rises sharply through the intermediate regime, and then enters a broad high-performing plateau at positive \(\lambda\).
This supports the claim that the low-noise advantage is robust rather than tied to a single finely tuned canonical level.
}
\label{fig:lambda_heatmaps}
\end{figure*}

Figure~\ref{fig:lambda_heatmaps} shows the same qualitative pattern on both backbones. For strongly negative \(\lambda\), both detectors remain close to chance, indicating that highly corrupted canonical levels are not useful. Performance then rises rapidly as \(\lambda\) moves toward lower-noise regimes, before flattening into a broad plateau for positive \(\lambda\). Thus, the benefit of low-noise probing is not confined to one operating point.

There are, however, mild backbone-specific differences. On improved-diffusion, \(\cfs(1{\times}2)\) is most helpful in the transition regime, whereas on the final high-\(\lambda\) plateau the gap to decoder-only probing becomes small. EDM shows an even flatter plateau once \(\lambda\) becomes positive. Overall, the sweeps suggest that the canonical level is the primary driver, whereas richer representation composition mainly improves the intermediate regime and worst-case robustness.

\subsection{\cfs family analysis}
\label{app:ablation_repr_comp}

The role of representation composition is already visible in the budget-matched rows of Table~\ref{tab:budget_nfe}, so we do not duplicate those results. At fixed logical cost, the comparison between \(\cfs(2\times1_{\mathrm{enc}})\), \(\cfs(2\times1_{\mathrm{dec}})\), and \(\cfs(2\times2)\) shows how performance changes when using encoder-only, decoder-only, or combined encoder-decoder snapshots across two canonical levels.

Decoder-only probing already captures most of the gain once the canonical level reaches the strong low-noise regime. Encoder-decoder composition is most useful before the final plateau is reached and for worst-case robustness. The main claim is therefore not that composition is always necessary, but that it provides a meaningful robustness margin beyond an already strong decoder-only baseline.

\subsection{Pooling rule}
\label{app:ablation_pooling}

We compare mean-only, std-only, and mean+std pooling for \(\cfs(1{\times}2)\). This tests whether first-order channel averages suffice or whether within-channel dispersion carries additional useful signal once both encoder and decoder snapshots are retained. Results are presented in Table~\ref{tab:pooling_ablation}.

\begin{table}[t]
\centering
\caption{
\textbf{Pooling rule ablation for \(\cfs(1{\times}2)\).}
We report both backbone families separately to test whether the gain of the primary detector relies only on first-order channel averages or also on within-channel dispersion captured by standard deviation pooling.
}
\label{tab:pooling_ablation}
\small
\setlength{\tabcolsep}{4.5pt}
\begin{tabular}{lcccc}
\toprule
Pooling
& \multicolumn{2}{c}{Improved-diffusion}
& \multicolumn{2}{c}{EDM} \\
\cmidrule(lr){2-3}\cmidrule(lr){4-5}
& \(\AvgAUROC \uparrow\) & \(\AvgWorstAUROC \uparrow\)
& \(\AvgAUROC \uparrow\) & \(\AvgWorstAUROC \uparrow\) \\
\midrule
Mean only   & 0.8645 & 0.7785 & 0.9078 & \textbf{0.8150} \\
Std only    & 0.8684 & 0.7381 & 0.9047 & 0.7693 \\
Mean + Std  & \textbf{0.8870} & \textbf{0.7990} & \textbf{0.9160} & 0.8140 \\
\bottomrule
\end{tabular}
\end{table}

Combining mean and standard-deviation pooling is clearly beneficial on improved-diffusion, where it gives the best average and worst-case performance. On EDM, mean+std gives the best average AUROC, while mean-only is marginally better in AvgWorstAUROC. Overall, dispersion features help the main operating point, but their gain is not perfectly uniform across criteria.

\subsection{Diagonal score versus alternative ID-only scores}
\label{app:ablation_heads}

The main paper uses a lightweight diagonal score so that the comparison stays focused on representation quality rather than downstream classifier capacity. We now test whether the strong performance of \(\cfs(1{\times}2)\) persists under stronger but still ID-only scores.

\begin{table}[t]
\centering
\caption{
\textbf{Head sensitivity under the same sparse representation for \(\cfs(1{\times}2)\).}
We report both backbone families separately.
The goal is to test whether the usefulness of the primary detector comes mainly from the selected representation rather than from a particular downstream head.
}
\label{tab:head_ablation}
\small
\setlength{\tabcolsep}{4.5pt}
\begin{tabular}{lcccc}
\toprule
Head
& \multicolumn{2}{c}{Improved-diffusion}
& \multicolumn{2}{c}{EDM} \\
\cmidrule(lr){2-3}\cmidrule(lr){4-5}
& \(\AvgAUROC \uparrow\) & \(\AvgWorstAUROC \uparrow\)
& \(\AvgAUROC \uparrow\) & \(\AvgWorstAUROC \uparrow\) \\
\midrule
Diagonal (main paper) & \second{0.8861} & 0.7973 & 0.9161 & 0.8122 \\
Shrinkage covariance  & 0.8763 & \second{0.8234} & \second{0.9207} & \second{0.8618} \\
KNN                   & \best{0.9101} & \best{0.8622} & \best{0.9501} & \best{0.9043} \\
GMM-light             & 0.8597 & 0.7838 & 0.9121 & 0.7837 \\
\bottomrule
\end{tabular}
\end{table}

Table~\ref{tab:head_ablation} shows that the sparse \(\cfs(1{\times}2)\) representation remains strong across several ID-only heads, which supports the view that the main signal comes from the selected features rather than from a specialized classifier. At the same time, the exact ranking is not head-invariant: KNN is strongest on both backbones, and shrinkage covariance also improves worst-case robustness over the diagonal head.

We therefore interpret the diagonal score as a conservative evaluation choice rather than as the empirically strongest one. Its role in the main paper is to keep the detector lightweight and to avoid conflating feature quality with additional head capacity.

\subsection{ID-fit bank size sensitivity}
\label{app:ablation_bank_size}

Because \(\cfs(1{\times}2)\) is calibrated from ID-only reference statistics, practical performance may depend on the size of the ID-fit bank. We therefore vary the number of ID-fit samples used to estimate the slot statistics and report the resulting performance.

\begin{table}[t]
\centering
\caption{
\textbf{Sensitivity to the ID-fit bank size for \(\cfs(1{\times}2)\).}
Each entry reports the mean over three random seeds.
The goal is to assess whether the sparse two-slot representation reaches stable performance with moderate ID-only calibration data.
}
\label{tab:bank_size_ablation}
\small
\setlength{\tabcolsep}{4.5pt}
\begin{tabular}{ccccc}
\toprule
ID-fit bank size
& \multicolumn{2}{c}{Improved-diffusion}
& \multicolumn{2}{c}{EDM} \\
\cmidrule(lr){2-3}\cmidrule(lr){4-5}
& \(\AvgAUROC \uparrow\) & \(\AvgWorstAUROC \uparrow\)
& \(\AvgAUROC \uparrow\) & \(\AvgWorstAUROC \uparrow\) \\
\midrule
100  & 0.8873 & 0.8031 & 0.9177 & 0.8207 \\
500  & 0.8871 & 0.8006 & 0.9164 & 0.8150 \\
1000 & 0.8863 & 0.7982 & 0.9161 & 0.8117 \\
Full & 0.8866 & 0.7979 & 0.9164 & 0.8127 \\
\bottomrule
\end{tabular}
\end{table}

Table~\ref{tab:bank_size_ablation} shows that \(\cfs(1{\times}2)\) is only weakly sensitive to the ID-fit bank size on both backbones. Even a bank of 100 ID samples already performs very close to the larger-bank regime. There is therefore no evidence that the main operating point depends on an unusually large calibration set.

\subsection{Statistical stability across random seeds}
\label{app:ablation_seeds}

Finally, we report repeated runs that vary the stochastic seed in the evaluation pipeline. The goal is not to estimate every source of variance exhaustively, but to check whether the main conclusions remain stable under the stochastic components that most directly affect the reported scores.

Tables~\ref{tab:seed_stability_improved} and~\ref{tab:seed_stability_edm} report seed stability for both \(\cfs(1{\times}2)\) and \(\cfs_{\mathrm{dec}}(1{\times}1)\).

\begin{table}[t]
\centering
\caption{
\textbf{Statistical stability across random seeds on improved-diffusion.}
We report mean \(\pm\) standard deviation over repeated runs.
The very small deviations indicate that the comparison between the two sparse CFS variants is not driven by stochastic fluctuations.
}
\label{tab:seed_stability_improved}
\small
\setlength{\tabcolsep}{4.5pt}
\begin{tabular}{lccc}
\toprule
Method & \(\AvgAUROC \uparrow\) & \(\AvgWorstAUROC \uparrow\) & FPR95 \(\downarrow\) \\
\midrule
\(\cfs_{\mathrm{dec}}(1{\times}1)\)   & \(0.88566 \pm 0.00008\) & \(0.79103 \pm 0.00015\) & \(0.35678 \pm 0.00016\) \\
\(\cfs(1{\times}2)\)                  & \(0.88609 \pm 0.00002\) & \(0.79721 \pm 0.00003\) & \(0.33546 \pm 0.00024\) \\
\bottomrule
\end{tabular}
\end{table}

\begin{table}[t]
\centering
\caption{
\textbf{Statistical stability across random seeds on EDM.}
We report mean \(\pm\) standard deviation over repeated runs.
Again, the deviations are extremely small, showing that the reported trade-off between the two sparse CFS variants is stable across seeds.
}
\label{tab:seed_stability_edm}
\small
\setlength{\tabcolsep}{4.5pt}
\begin{tabular}{lccc}
\toprule
Method & \(\AvgAUROC \uparrow\) & \(\AvgWorstAUROC \uparrow\) & FPR95 \(\downarrow\) \\
\midrule
\(\cfs_{\mathrm{dec}}(1{\times}1)\)   & \(0.91950 \pm 0.00009\) & \(0.80730 \pm 0.00027\) & \(0.18149 \pm 0.00021\) \\
\(\cfs(1{\times}2)\)                  & \(0.91618 \pm 0.00009\) & \(0.81223 \pm 0.00005\) & \(0.19132 \pm 0.00012\) \\
\bottomrule
\end{tabular}
\end{table}

Across both backbones, the standard deviations are tiny, so the observed rankings are highly stable across repeated runs. For improved-diffusion, \(\cfs(1{\times}2)\) consistently remains slightly stronger than \(\cfs_{\mathrm{dec}}(1{\times}1)\) in both average and worst-case AUROC while also improving FPR95. For EDM, the trade-off is equally stable but slightly different: \(\cfs_{\mathrm{dec}}(1{\times}1)\) retains a small advantage in \(\AvgAUROC\) and FPR95, whereas \(\cfs(1{\times}2)\) retains the better \(\AvgWorstAUROC\).

\subsection{Architecture-transfer sanity check}
\label{app:arch_transfer}

Although our main experiments use U-Net-style diffusion backbones, CFS only requires access to block-level internal activations at canonical corruption levels. As a sanity check, we apply the same sparse snapshot construction to a transformer-based diffusion backbone, U-ViT. Unlike convolutional U-Nets, U-ViT keeps a token sequence of nearly constant size across its transformer blocks; its U-shaped structure comes from input blocks, a middle block, output blocks, and long skip connections rather than from explicit spatial downsampling and upsampling. We therefore replace the U-Net encoder/decoder hook taxonomy by early, middle, and late transformer-block snapshots, and pool token features using channel-wise mean and standard deviation over image tokens.

The purpose is to test whether the sparse-snapshot principle transfers beyond convolutional U-Nets. All runs use the same low-noise canonical level, with target $\lambda=5.0$, effective $\lambda=5.0329$, native U-ViT timestep $t=21$, and $b^2=6.48\times 10^{-3}$. The hooked U-ViT activations have shape $B\times 257\times 512$, corresponding to 256 image tokens plus one time token; CFS pools the image tokens only. We also include an output-space baseline using the same U-ViT checkpoint and the same canonical level: it scores images by the final reconstruction MSE $\|\hat{\mathbf{x}}_0-\mathbf{x}_0\|^2$ with an ID-only one-dimensional density fit.

\begin{table}[t]
\centering
\caption{\textbf{Architecture-transfer sanity check on U-ViT.}
CFS probes early, middle, and late U-ViT block outputs instead of U-Net encoder/decoder hooks. All variants use one canonical low-noise level, one backbone forward, token mean/std pooling, and an ID-only diagonal score. The output baseline uses the same checkpoint, level, and one-forward budget, but scores only the final reconstruction error. Averages are computed over 12 ordered pairs: for each ID dataset in \{CIFAR-10, SVHN, CelebA32\}, we evaluate against the other two source datasets plus CIFAR-100 and DTD as OOD.}
\label{tab:arch_transfer_transformer}
\small
\begin{tabular}{llcccc}
\toprule
Backbone & Method & Region / signal & AvgAUROC $\uparrow$ & AvgFPR95 $\downarrow$ & Cost $\downarrow$ \\
\midrule
U-ViT-CIFAR10 & Output recon. MSE & output & 0.778 & 0.653 & 1F \\
U-ViT-CIFAR10 & CFS-early & \texttt{in\_blocks.2} & 0.872 & 0.356 & 1F \\
U-ViT-CIFAR10 & CFS-mid & \texttt{mid\_block} & 0.874 & 0.292 & 1F \\
U-ViT-CIFAR10 & CFS-late & \texttt{out\_blocks.2} & \second{0.891} & \best{0.213} & 1F \\
U-ViT-CIFAR10 & CFS-early+late & \texttt{in\_blocks.2}+\texttt{out\_blocks.2} & \best{0.896} & \second{0.238} & 1F \\
\bottomrule
\end{tabular}
\end{table}

Tables~\ref{tab:arch_transfer_transformer} and~\ref{tab:arch_transfer_transformer_breakdown} support the architecture-transfer interpretation: sparse internal snapshots remain informative even when spatial U-Net feature maps are replaced by token-level transformer states. Importantly, the gain does not come merely from using a low-noise U-ViT forward: the output reconstruction baseline reaches only $0.778$ AvgAUROC and $0.653$ AvgFPR95, whereas CFS-late improves to $0.891$ AvgAUROC and $0.213$ AvgFPR95 at the same one-forward cost. The paired early+late variant slightly improves AvgAUROC, while the late snapshot gives the best AvgFPR95, suggesting that late transformer blocks already concentrate most of the useful signal in this setting. We emphasize that this result should not be read as a new U-ViT OOD benchmark: performance is not uniform across all source datasets, and the CIFAR-10-as-ID case remains the main failure mode.

\begin{table}[t]
\centering
\caption{\textbf{Per-source breakdown of the U-ViT architecture-transfer sanity check.}
Each cell reports the average over the four OOD datasets associated with the corresponding ID source. The output baseline uses final reconstruction MSE, while CFS variants use internal token snapshots. The main failure mode is CIFAR-10 as ID, especially against CelebA32 and CIFAR-100.}
\label{tab:arch_transfer_transformer_breakdown}
\small
\begin{tabular}{llccc}
\toprule
Method & Metric & ID=SVHN & ID=CIFAR-10 & ID=CelebA32 \\
\midrule
Output recon. MSE & AvgAUROC $\uparrow$ & 0.927 & 0.594 & 0.813 \\
Output recon. MSE & AvgFPR95 $\downarrow$ & 0.467 & 0.803 & 0.687 \\
\midrule
CFS-early & AvgAUROC $\uparrow$ & 0.945 & \best{0.722} & 0.949 \\
CFS-early & AvgFPR95 $\downarrow$ & 0.238 & 0.638 & 0.191 \\
\midrule
CFS-mid & AvgAUROC $\uparrow$ & 0.979 & 0.652 & \second{0.991} \\
CFS-mid & AvgFPR95 $\downarrow$ & 0.086 & 0.747 & 0.044 \\
\midrule
CFS-late & AvgAUROC $\uparrow$ & \best{0.996} & 0.686 & \best{0.992} \\
CFS-late & AvgFPR95 $\downarrow$ & \best{0.016} & \second{0.587} & \best{0.036} \\
\midrule
CFS-early+late & AvgAUROC $\uparrow$ & \second{0.985} & \second{0.719} & 0.984 \\
CFS-early+late & AvgFPR95 $\downarrow$ & \second{0.072} & \best{0.575} & \second{0.066} \\
\bottomrule
\end{tabular}
\end{table}

\section{Extended CIFAR-Scale Results}
\label{app:cifar}

This section provides the full pairwise breakdown behind the compact summaries shown in the main paper. Its role is descriptive rather than argumentative: it exposes where the main representation-space advantage comes from, pair by pair, under the same controlled protocol.



We begin with the full pairwise breakdown for the primary shared-source benchmark, i.e.\ the CIFAR10-source checkpoint family used in the main paper. Tables~\ref{tab:cifar10_source_full_auroc} and~\ref{tab:cifar10_source_full_fpr95} report AUROC and FPR95 for all ID\(\to\)OOD pairs. To keep the appendix readable, we expose pairwise operating-point behavior mainly for the two main sparse variants, \(\cfs_{\mathrm{dec}}(1{\times}1)\) and \(\cfs(1{\times}2)\).

\begin{table*}[t]
\centering
\caption{
\textbf{Full per-pair AUROC on the CIFAR-scale benchmark under the primary CIFAR10-source policy.}
Each entry corresponds to one ID\(\to\)OOD pair.
The same fixed shared-source protocol is used for all methods.
All metrics are reported with three decimal places.
}
\label{tab:cifar10_source_full_auroc}
\scriptsize
\resizebox{\textwidth}{!}{%
\begin{tabular}{llcccc|cccc|cccc}
\toprule
& & \multicolumn{4}{c|}{CIFAR10 as ID}
  & \multicolumn{4}{c|}{SVHN as ID}
  & \multicolumn{4}{c}{CelebA32 as ID} \\
\cmidrule(lr){3-6}\cmidrule(lr){7-10}\cmidrule(lr){11-14}
Method & Backbone
& SVHN & CelebA32 & CIFAR100 & DTD
& CIFAR10 & CelebA32 & CIFAR100 & DTD
& CIFAR10 & SVHN & CIFAR100 & DTD \\
\midrule
\msma     & improved & 0.646 & 0.438 & 0.579 & 0.769
                    & 0.916 & 0.925 & 0.935 & 0.936
                    & 0.710 & 0.950 & 0.762 & 0.936 \\
\msma     & EDM      & 0.663 & 0.490 & 0.606 & 0.796
                    & 0.937 & 0.962 & 0.936 & \second{0.941}
                    & 0.640 & 0.933 & 0.738 & 0.908 \\
\addlinespace
\diffpath & improved & 0.886 & 0.371 & 0.550 & 0.683
                    & 0.940 & 0.965 & 0.946 & 0.914
                    & 0.640 & 0.955 & 0.688 & 0.797 \\
\diffpath & EDM      & 0.832 & 0.316 & 0.561 & 0.780
                    & 0.903 & 0.960 & 0.891 & 0.924
                    & 0.697 & 0.970 & 0.763 & 0.911 \\
\addlinespace
\ddpmood  & improved & 0.132 & 0.527 & 0.538 & 0.433
                    & 0.894 & 0.947 & 0.904 & 0.756
                    & 0.476 & 0.059 & 0.522 & 0.415 \\
\ddpmood  & EDM      & 0.124 & \best{0.586} & 0.568 & 0.431
                    & 0.907 & 0.965 & 0.926 & 0.802
                    & 0.441 & 0.035 & 0.516 & 0.409 \\
\addlinespace
\gepc     & improved & 0.584 & 0.487 & 0.494 & 0.600
                    & 0.754 & 0.580 & 0.595 & 0.584
                    & 0.736 & 0.659 & 0.571 & 0.747 \\
\gepc     & EDM      & 0.917 & 0.317 & 0.506 & 0.680
                    & 0.963 & 0.954 & 0.855 & 0.798
                    & 0.799 & 0.962 & 0.686 & 0.848 \\
\midrule
\textbf{\(\cfs_{\mathrm{dec}}(1{\times}1)\)} & improved & \second{0.925} & \second{0.540} & \second{0.687} & 0.941
                                           & 0.964 & 0.897 & 0.873 & 0.893
                                           & 0.983 & 0.973 & 0.966 & 0.990 \\
\textbf{\(\cfs_{\mathrm{dec}}(1{\times}1)\)} & EDM      & \best{0.969} & 0.451 & \best{0.697} & \best{0.968}
                                           & \best{0.997} & \best{0.993} & \second{0.986} & \best{0.991}
                                           & \best{0.993} & \best{0.999} & \best{0.990} & \best{0.999} \\
\textbf{\(\cfs(1{\times}2)\)}     & improved & 0.888 & 0.524 & 0.635 & 0.932
                                           & 0.966 & 0.947 & 0.907 & 0.917
                                           & 0.981 & \second{0.984} & 0.966 & 0.993 \\
\textbf{\(\cfs(1{\times}2)\)}     & EDM      & \best{0.969} & 0.468 & 0.665 & \second{0.959}
                                           & \second{0.996} & \second{0.990} & \best{0.989} & \best{0.991}
                                           & \second{0.988} & \best{0.999} & \second{0.986} & \second{0.998} \\
\bottomrule
\end{tabular}%
}
\end{table*}

\begin{table*}[t]
\centering
\caption{
\textbf{Full per-pair FPR95 on the CIFAR-scale benchmark under the primary CIFAR10-source policy.}
Each entry corresponds to one ID\(\to\)OOD pair.
Lower is better.
All metrics are reported with three decimal places.
}
\label{tab:cifar10_source_full_fpr95}
\scriptsize
\resizebox{\textwidth}{!}{%
\begin{tabular}{llcccc|cccc|cccc}
\toprule
& & \multicolumn{4}{c|}{CIFAR10 as ID}
  & \multicolumn{4}{c|}{SVHN as ID}
  & \multicolumn{4}{c}{CelebA32 as ID} \\
\cmidrule(lr){3-6}\cmidrule(lr){7-10}\cmidrule(lr){11-14}
Method & Backbone
& SVHN & CelebA32 & CIFAR100 & DTD
& CIFAR10 & CelebA32 & CIFAR100 & DTD
& CIFAR10 & SVHN & CIFAR100 & DTD \\
\midrule
\msma     & improved & 0.828 & 0.966 & 0.925 & 0.780
                    & 0.416 & 0.310 & 0.367 & 0.387
                    & 0.872 & 0.273 & 0.825 & 0.412 \\
\msma     & EDM      & 0.839 & 0.945 & 0.908 & 0.782
                    & 0.339 & 0.169 & 0.356 & 0.380
                    & 0.904 & 0.266 & 0.841 & 0.574 \\
\addlinespace
\diffpath & improved & 0.335 & 0.969 & 0.931 & 0.882
                    & 0.381 & 0.185 & 0.370 & 0.590
                    & 0.897 & 0.188 & 0.868 & 0.745 \\
\diffpath & EDM      & 0.435 & 0.979 & 0.927 & 0.785
                    & 0.641 & 0.206 & 0.704 & 0.539
                    & 0.891 & 0.139 & 0.840 & 0.553 \\
\addlinespace
\ddpmood  & improved & 0.996 & 0.853 & 0.948 & 0.996
                    & 0.548 & 0.229 & 0.523 & 0.969
                    & 0.993 & 1.000 & 0.993 & 1.000 \\
\ddpmood  & EDM      & 0.996 & \best{0.809} & 0.932 & 0.993
                    & 0.468 & 0.134 & 0.387 & 0.900
                    & 0.998 & 1.000 & 0.996 & 1.000 \\
\addlinespace
\gepc     & improved & 0.826 & 0.948 & 0.949 & 0.847
                    & 0.714 & 0.863 & 0.874 & 0.892
                    & 0.793 & 0.832 & 0.903 & 0.712 \\
\gepc     & EDM      & 0.330 & 0.977 & 0.947 & 0.902
                    & 0.241 & 0.261 & 0.672 & 0.851
                    & 0.801 & 0.188 & 0.856 & 0.683 \\
\midrule
\textbf{\(\cfs_{\mathrm{dec}}(1{\times}1)\)} & improved & 0.269 & \second{0.835} & \best{0.803} & 0.271
                                           & 0.230 & 0.431 & 0.523 & 0.487
                                           & 0.095 & 0.126 & 0.149 & 0.048 \\
\textbf{\(\cfs_{\mathrm{dec}}(1{\times}1)\)} & EDM      & \best{0.093} & 0.868 & \second{0.849} & \best{0.133}
                                           & \best{0.011} & \best{0.033} & \second{0.064} & \best{0.043}
                                           & \best{0.036} & \best{0.003} & \best{0.046} & \best{0.003} \\
\textbf{\(\cfs(1{\times}2)\)}     & improved & 0.371 & 0.842 & 0.855 & 0.329
                                           & 0.205 & 0.216 & 0.409 & 0.404
                                           & 0.109 & 0.080 & 0.157 & 0.037 \\
\textbf{\(\cfs(1{\times}2)\)}     & EDM      & \second{0.098} & 0.862 & 0.861 & \second{0.177}
                                           & \second{0.021} & \second{0.042} & \best{0.052} & \second{0.048}
                                           & \second{0.057} & \second{0.004} & \second{0.065} & \second{0.007} \\
\bottomrule
\end{tabular}%
}
\end{table*}

The full tables make two patterns clear. First, the main-paper summaries are not driven by one isolated pair: the sparse \cfs\ variants are strong across a broad set of ID\(\to\)OOD configurations, especially on compact IDs and texture-driven shifts. Second, the decoder-only and paired encoder-decoder variants remain very close on many pairs, which reinforces the compression result already visible in the main tables.

\section{Full Source-Family Robustness Results}
\label{app:source_family_full}

This section tests whether the main representation-space pattern is tied to one checkpoint family or survives a controlled change in the frozen source representation.



We report the full pairwise breakdown for the alternative CelebA32-source checkpoint family. Tables~\ref{tab:celeba_source_full_auroc} and~\ref{tab:celeba_source_full_fpr95} provide the detailed AUROC and FPR95 results. The goal is to test whether the source-family change modifies only absolute performance or also the qualitative ranking between representation-space and output-space methods.

\begin{table*}[t]
\centering
\caption{
\textbf{Full per-pair AUROC on the CIFAR-scale benchmark under the alternative CelebA32-source policy.}
This is the detailed source-family robustness breakdown complementary to the main-paper summary.
All metrics are reported with three decimal places.
}
\label{tab:celeba_source_full_auroc}
\scriptsize
\resizebox{\textwidth}{!}{%
\begin{tabular}{llcccc|cccc|cccc}
\toprule
& & \multicolumn{4}{c|}{CIFAR10 as ID}
  & \multicolumn{4}{c|}{SVHN as ID}
  & \multicolumn{4}{c}{CelebA32 as ID} \\
\cmidrule(lr){3-6}\cmidrule(lr){7-10}\cmidrule(lr){11-14}
Method & Backbone
& SVHN & CelebA32 & CIFAR100 & DTD
& CIFAR10 & CelebA32 & CIFAR100 & DTD
& CIFAR10 & SVHN & CIFAR100 & DTD \\
\midrule
\msma     & improved & 0.766 & 0.859 & 0.562 & 0.781
                    & 0.937 & 0.976 & 0.944 & \second{0.934}
                    & \second{0.956} & 0.930 & \second{0.953} & 0.973 \\
\msma     & EDM      & 0.659 & 0.278 & 0.536 & 0.705
                    & 0.928 & 0.937 & 0.931 & 0.891
                    & 0.851 & \second{0.984} & 0.860 & 0.916 \\
\addlinespace
\diffpath & improved & \best{0.952} & 0.641 & 0.538 & 0.700
                    & \best{0.979} & 0.946 & \best{0.970} & 0.913
                    & 0.779 & 0.915 & 0.787 & 0.827 \\
\diffpath & EDM      & 0.817 & 0.350 & 0.524 & 0.726
                    & 0.936 & 0.957 & 0.933 & 0.922
                    & 0.680 & 0.900 & 0.709 & 0.864 \\
\addlinespace
\ddpmood  & improved & 0.094 & 0.201 & 0.482 & 0.349
                    & 0.928 & 0.761 & 0.914 & 0.741
                    & 0.841 & 0.251 & 0.816 & 0.595 \\
\ddpmood  & EDM      & 0.069 & 0.242 & 0.503 & 0.348
                    & \second{0.955} & 0.914 & \second{0.950} & 0.809
                    & 0.750 & 0.097 & 0.741 & 0.520 \\
\addlinespace
\gepc     & improved & 0.533 & \best{0.975} & 0.508 & 0.643
                    & 0.516 & 0.963 & 0.515 & 0.577
                    & 0.929 & 0.944 & 0.928 & 0.957 \\
\gepc     & EDM      & 0.728 & 0.545 & 0.533 & 0.654
                    & 0.910 & 0.966 & 0.913 & 0.833
                    & 0.577 & 0.925 & 0.590 & 0.765 \\
\midrule
\textbf{\(\cfs_{\mathrm{dec}}(1{\times}1)\)} & improved & 0.827 & 0.943 & 0.604 & 0.865
                                           & 0.884 & \best{0.997} & 0.905 & 0.867
                                           & \best{0.999} & \best{1.000} & \best{0.998} & \best{1.000} \\
\textbf{\(\cfs_{\mathrm{dec}}(1{\times}1)\)} & EDM      & 0.858 & 0.953 & \second{0.605} & 0.863
                                           & 0.935 & 0.991 & 0.950 & \second{0.962}
                                           & 0.941 & 0.976 & 0.947 & 0.985 \\
\textbf{\(\cfs(1{\times}2)\)}     & improved & 0.828 & 0.907 & 0.602 & \second{0.878}
                                           & 0.903 & \second{0.993} & 0.916 & 0.881
                                           & \second{0.997} & \best{1.000} & \second{0.996} & \second{0.999} \\
\textbf{\(\cfs(1{\times}2)\)}     & EDM      & \second{0.879} & \second{0.957} & \best{0.618} & \best{0.887}
                                           & \second{0.955} & \second{0.993} & \second{0.966} & \best{0.974}
                                           & 0.963 & \second{0.989} & 0.968 & 0.993 \\
\bottomrule
\end{tabular}%
}
\end{table*}

\begin{table*}[t]
\centering
\caption{
\textbf{Full per-pair FPR95 on the CIFAR-scale benchmark under the alternative CelebA32-source policy.}
This is the detailed source-family robustness breakdown complementary to the main-paper summary.
Lower is better.
All metrics are reported with three decimal places.
}
\label{tab:celeba_source_full_fpr95}
\scriptsize
\resizebox{\textwidth}{!}{%
\begin{tabular}{llcccc|cccc|cccc}
\toprule
& & \multicolumn{4}{c|}{CIFAR10 as ID}
  & \multicolumn{4}{c|}{SVHN as ID}
  & \multicolumn{4}{c}{CelebA32 as ID} \\
\cmidrule(lr){3-6}\cmidrule(lr){7-10}\cmidrule(lr){11-14}
Method & Backbone
& SVHN & CelebA32 & CIFAR100 & DTD
& CIFAR10 & CelebA32 & CIFAR100 & DTD
& CIFAR10 & SVHN & CIFAR100 & DTD \\
\midrule
\msma     & improved & 0.627 & 0.520 & 0.941 & 0.788
                    & 0.380 & 0.121 & 0.335 & 0.423
                    & 0.233 & 0.298 & 0.251 & 0.135 \\
\msma     & EDM      & 0.768 & 0.990 & 0.941 & 0.872
                    & 0.362 & 0.261 & 0.355 & 0.593
                    & 0.618 & 0.069 & 0.600 & 0.460 \\
\addlinespace
\diffpath & improved & \best{0.118} & 0.848 & 0.948 & 0.875
                    & \best{0.057} & 0.312 & \second{0.141} & 0.674
                    & 0.849 & 0.293 & 0.845 & 0.820 \\
\diffpath & EDM      & \second{0.351} & 0.975 & 0.940 & 0.813
                    & 0.399 & 0.181 & 0.434 & 0.470
                    & 0.886 & 0.235 & 0.877 & 0.668 \\
\addlinespace
\ddpmood  & improved & 0.999 & 0.986 & 0.962 & 0.999
                    & 0.370 & 0.681 & 0.456 & 0.960
                    & 0.768 & 0.998 & 0.837 & 0.999 \\
\ddpmood  & EDM      & 1.000 & 0.977 & 0.953 & 0.999
                    & 0.220 & 0.280 & 0.243 & 0.897
                    & 0.901 & 1.000 & 0.925 & 1.000 \\
\addlinespace
\gepc     & improved & 0.922 & 0.193 & 0.948 & 0.864
                    & 0.937 & 0.202 & 0.943 & 0.914
                    & 0.175 & 0.137 & 0.188 & 0.135 \\
\gepc     & EDM      & 0.858 & 0.918 & 0.937 & 0.906
                    & 0.425 & 0.151 & 0.433 & 0.757
                    & 0.929 & 0.481 & 0.910 & 0.781 \\
\midrule
\textbf{\(\cfs_{\mathrm{dec}}(1{\times}1)\)} & improved & 0.532 & \second{0.154} & 0.918 & 0.599
                                           & 0.484 & \best{0.014} & 0.438 & 0.580
                                           & \best{0.004} & \best{0.001} & \best{0.007} & \best{0.001} \\
\textbf{\(\cfs_{\mathrm{dec}}(1{\times}1)\)} & EDM      & 0.378 & \second{0.143} & \best{0.903} & \second{0.545}
                                           & 0.294 & 0.023 & 0.231 & \second{0.158}
                                           & 0.146 & 0.065 & 0.140 & 0.053 \\
\textbf{\(\cfs(1{\times}2)\)}     & improved & 0.538 & 0.218 & 0.917 & 0.557
                                           & 0.385 & 0.028 & 0.371 & 0.487
                                           & \second{0.013} & \second{0.002} & \second{0.017} & \second{0.002} \\
\textbf{\(\cfs(1{\times}2)\)}     & EDM      & 0.354 & \best{0.133} & \second{0.909} & \best{0.501}
                                           & \second{0.178} & \second{0.017} & \best{0.140} & \best{0.112}
                                           & \second{0.101} & \second{0.037} & \second{0.098} & \second{0.033} \\
\bottomrule
\end{tabular}%
}
\end{table*}


A source-family change modifies the representation bias of the frozen checkpoint. The relevant question is therefore not whether absolute numbers remain identical, but whether the ranking and the representation-versus-output-space contrast remain qualitatively stable.

The full tables support the same broad conclusion as the compact main-paper summary: changing the source family shifts absolute numbers but does not erase the main pattern. Sparse representation-space probing remains highly competitive, and the contrast between \cfs\ and broader output-space baselines survives the change in frozen source representation.

\section{External Positioning Against Prior Reported Diffusion Results}
\label{app:external_positioning}

This section provides external context by comparing \cfs\ to previously reported diffusion-based CIFAR-scale results. 

\begin{table*}[t]
\centering
\caption{
\textbf{External positioning against prior reported diffusion-based results on the CIFAR-scale benchmark.}
Rows above the final block are taken from prior work. We separate: (i) prior diffusion-based methods reported in ID-specific settings, and (ii) single-checkpoint diffusion methods using one frozen checkpoint across multiple ID/OOD pairs.
}
\label{tab:external_positioning}
\footnotesize
\setlength{\tabcolsep}{3.2pt}
\renewcommand{\arraystretch}{1.04}
\resizebox{\textwidth}{!}{%
\begin{tabular}{lcccccccccccccc}
\toprule
& \multicolumn{4}{c}{CIFAR-10 as ID}
& \multicolumn{4}{c}{SVHN as ID}
& \multicolumn{4}{c}{CelebA32 as ID}
& \multirow{2}{*}{Avg.}
& \multirow{2}{*}{$\Cost$} \\
\cmidrule(lr){2-5}\cmidrule(lr){6-9}\cmidrule(lr){10-13}
Method
& SVHN & CelebA & C100 & DTD
& C10 & CelebA & C100 & DTD
& C10 & SVHN & C100 & DTD
&  &  \\
\midrule
\multicolumn{15}{c}{\textit{Diffusion-based, ID-specific / reported from prior work}} \\
\midrule
NLL
& 0.091 & 0.574 & 0.521 & 0.609
& \second{0.990} & \second{0.999} & \best{0.992} & 0.983
& 0.814 & 0.105 & 0.786 & 0.809
& 0.689 & 1000F \\

IC
& 0.921 & 0.516 & 0.519 & 0.553
& 0.080 & 0.028 & 0.100 & 0.174
& 0.485 & 0.972 & 0.510 & 0.559
& 0.451 & 1000F \\

MSMA
& 0.957 & \best{1.000} & 0.615 & \second{0.986}
& 0.976 & 0.995 & 0.980 & \second{0.996}
& 0.910 & \second{0.996} & 0.927 & \second{0.999}
& \second{0.945} & 10F \\

DDPM-OOD
& 0.390 & 0.659 & 0.536 & 0.598
& 0.951 & 0.986 & 0.945 & 0.910
& 0.795 & 0.636 & 0.778 & 0.773
& 0.746 & 350F \\

LMD
& \best{0.992} & 0.557 & 0.604 & 0.667
& 0.919 & 0.890 & 0.881 & 0.914
& 0.989 & \best{1.000} & 0.979 & 0.972
& 0.865 & $10^4$F \\

EigenScore
& 0.810 & 0.873 & \best{0.880} & --
& \best{0.992} & 0.994 & \second{0.982} & --
& 0.965 & 0.888 & 0.944 & --
& 0.925 & 300F \\
\midrule
\multicolumn{15}{c}{\textit{Diffusion-based, single-checkpoint / outside \mbe\ (mostly reported from prior work)}} \\
\midrule

SCOPED-CelebA
& 0.814 & 0.940 & 0.477 & --
& 0.971 & 0.996 & 0.959 & --
& 0.925 & 0.994 & 0.962 & --
& 0.892 & 2F+2J \\
GEPC-CelebA
& 0.842 & \second{0.999} & 0.554 & --
& 0.880 & \best{1.000} & 0.897 & --
& \best{1.000} & \best{1.000} & \best{1.000} & --
& 0.908 & 8F \\

\addlinespace
DiffPath-6D-CelebA
& 0.910 & 0.897 & 0.590 & 0.923
& 0.939 & 0.979 & 0.953 & 0.981
& 0.998 & \best{1.000} & 0.998 & \second{0.999}
& 0.931 & 10F \\

DiffPath-6D-ImageNet
& 0.856 & 0.502 & 0.580 & 0.841
& 0.943 & 0.964 & 0.954 & 0.969
& 0.807 & 0.981 & 0.843 & 0.964
& 0.850 & 10F \\

\addlinespace
\textbf{\(\cfs(1{\times}2)\)-CelebA (ours)}
& 0.901 & 0.939 & 0.614 & 0.900
& 0.952 & 0.998 & 0.960 & 0.957
& \second{0.999} & \best{1.000} & \second{0.999} & \best{1.000}
& 0.935 & \best{1F} \\

\textbf{\(\cfs(1{\times}2)\)-ImageNet (ours)}
& \second{0.962} & 0.994 & \second{0.629} & \best{0.996}
& 0.981 & \best{1.000} & \second{0.982} & \best{1.000}
& \second{0.999} & \best{1.000} & \second{0.999} & 0.998
& \best{0.962} & \best{1F} \\
\bottomrule
\end{tabular}}
\end{table*}

Table~\ref{tab:external_positioning} shows that the sparse representation-space advantage of \cfs\ is not confined to the controlled \mbe\ setting. Even outside backbone-equated evaluation, \cfs\ remains competitive with or stronger than previously reported single-checkpoint diffusion results, while using only one backbone evaluation per image.

\subsection{Positioning against DLSR on its native published settings}
\label{app:dlsr_positioning}

DLSR is the closest prior in spirit, but it studies a different object: learned feature reconstruction rather than sparse probing of native frozen states. It also falls outside our MBE protocol. Since DLSR is only available on the ID/OOD pairs supported by its published setup, we report in Table~\ref{tab:dlsr_overlap} a positioning comparison on the DLSR-native overlapping subset, without treating this table as part of the main MBE claim.

\begin{table*}[t]
\centering
\resizebox{\textwidth}{!}{%
\begin{tabular}{lcccc|cccc|cccc|c|c}
\toprule
& \multicolumn{4}{c}{CIFAR-10 as ID}
& \multicolumn{4}{c}{SVHN as ID}
& \multicolumn{4}{c}{CelebA32 as ID}
& \multirow{2}{*}{Avg.\ overlap}
& \multirow{2}{*}{Cost} \\
\cmidrule(lr){2-5}\cmidrule(lr){6-9}\cmidrule(lr){10-13}
Method
& SVHN & CelebA & C100 & DTD
& C10 & CelebA & C100 & DTD
& C10 & SVHN & C100 & DTD
& & \\
\midrule

DLSR (+MSE)
& 0.973 & -- & \best{0.875} & \best{1.000}
& -- & -- & -- & --
& -- & -- & -- & \second{0.999}
& \second{0.962}
& -- \\

DLSR (+LR)
& \second{0.982} & -- & \second{0.872} & \best{1.000}
& -- & -- & -- & --
& -- & -- & -- & 0.985
& 0.960
& -- \\

DLSR (+MFsim)
& \best{0.989} & -- & 0.856 & \best{1.000}
& -- & -- & -- & --
& -- & -- & -- & \best{1.000}
& 0.961
& -- \\

\textbf{\(\cfs(1{\times}2)\)-ImageNet (diag)}
& 0.962 & \second{0.994} & 0.629 & \second{0.996}
& \second{0.981} & \best{1.000} & \second{0.982} & \best{1.000}
& \best{0.999} & \best{1.000} & \best{0.999} & 0.998
& \second{0.962}
& \best{1F} \\

\textbf{\(\cfs(1{\times}2)\)-ImageNet (kNN)}
& 0.977 & \best{1.000} & 0.717 & \best{1.000}
& \best{0.998} & \best{1.000} & \best{0.998} & \best{1.000}
& \best{0.999} & \best{1.000} & \best{0.999} & \second{0.999}
& \best{0.974}
& \best{1F} \\

\bottomrule
\end{tabular}%
}
\caption{\textbf{Positioning against DLSR on its native published evaluation pairs.} DLSR lies outside the MBE protocol and uses an additional learned feature-reconstruction module, so we do not report a backbone-forward-equivalent logical cost for it. For our methods, Cost denotes the backbone-forward logical cost. Averages are computed over the non-missing overlapping entries of each method.}
\label{tab:dlsr_overlap}
\end{table*}

DLSR is therefore evaluated only on pairs supported by the original protocol/repo. The full cross-ID benchmark would require extending the original DLSR training and evaluation pipeline beyond the published setup.

\section{Checkpoint-Controlled Large-Scale Results}
\label{app:large_scale}

This appendix reports a checkpoint-controlled large-scale comparison on ImageNet200 and ImageNet1K using a single official ImageNet-64 improved-diffusion checkpoint. These experiments are not meant to replace the controlled CIFAR-scale evidence in the main paper. Their role is narrower: to test whether the same representation-first signal remains informative in a substantially harder large-scale regime under a fixed backbone.

In all experiments below, all methods use the same official ImageNet-64 improved-diffusion checkpoint (\texttt{imagenet64\_uncond\_100M\_1500K.pt}). We evaluate on two near-OOD regimes, \textbf{NINCO} and \textbf{SSB-hard}, and one far-OOD regime, \textbf{Textures}. Unless stated otherwise, both \cfs\ variants use a single low-noise canonical level and therefore retain a logical test-time cost of \(1F+0J\), whereas \msma\ and \diffpath\ use \(10F+0J\).


Table~\ref{tab:large_scale_improved} reports the resulting large-scale comparison. In addition to the primary main-paper operating point \(\cfs(1{\times}2)\), we also include the ultra-sparse decoder-only companion \(\cfs_{\mathrm{dec}}(1{\times}1)\).

\begin{table*}[t]
\centering
\caption{
\textbf{Checkpoint-controlled large-scale comparison on the official ImageNet-64 improved-diffusion backbone.}
All methods use the same official improved-diffusion checkpoint (\texttt{imagenet64\_uncond\_100M\_1500K.pt}).
We report AUROC on NINCO, SSB-hard, and Textures, together with averaged AUROC, AUPR, and FPR95 over these three OOD sets.
The logical test-time cost of the displayed \cfs\ variants is \(1F\); for \msma\ and \diffpath\ it is \(10F\).
}
\label{tab:large_scale_improved}
\small
\setlength{\tabcolsep}{4.5pt}
\resizebox{\textwidth}{!}{%
\begin{tabular}{llccccccc}
\toprule
ID & Method
& NINCO $\uparrow$
& SSB-hard $\uparrow$
& Textures $\uparrow$
& Avg.\ AUROC $\uparrow$
& Avg.\ AUPR $\uparrow$
& Avg.\ FPR95 $\downarrow$
& Cost \\
\midrule
ImageNet200 & \msma
& 0.6121 & 0.5244 & 0.7804 & 0.6390 & 0.9651 & 0.8917 & \(\second{10F}\) \\
ImageNet200 & \diffpath
& 0.5874 & 0.5052 & 0.6325 & 0.5750 & 0.9525 & 0.9333 & \(\second{10F}\) \\
\addlinespace
ImageNet200 & \textbf{\(\cfs_{\mathrm{dec}}(1{\times}1)\)}
& \best{0.6549} & \second{0.5357} & \best{0.9093} & \best{0.7000} & \best{0.9739} & \best{0.7444} & \(\best{1F}\) \\
ImageNet200 & \textbf{\(\cfs(1{\times}2)\)}
& \second{0.6492} & \best{0.5361} & \second{0.8969} & \second{0.6941} & \second{0.9732} & \second{0.7578} & \(\best{1F}\) \\
\midrule
ImageNet1K & \msma
& 0.5972 & 0.5100 & 0.7790 & 0.6287 & 0.8513 & 0.8871 & \(\second{10F}\) \\
ImageNet1K & \diffpath
& 0.5913 & 0.5117 & 0.6359 & 0.5796 & 0.8130 & 0.9187 & \(\second{10F}\) \\
\addlinespace
ImageNet1K & \textbf{\(\cfs_{\mathrm{dec}}(1{\times}1)\)}
& \best{0.6613} & \best{0.5337} & \best{0.9081} & \best{0.7010} & \best{0.8902} & \best{0.7401} & \(\best{1F}\) \\
ImageNet1K & \textbf{\(\cfs(1{\times}2)\)}
& \second{0.6490} & \second{0.5247} & \second{0.8948} & \second{0.6895} & \second{0.8855} & \second{0.7591} & \(\best{1F}\) \\
\bottomrule
\end{tabular}}
\end{table*}

These numbers constitute a checkpoint-controlled stress test of whether the representation-first signal remains informative outside the controlled CIFAR-scale regime

This checkpoint-controlled large-scale view yields a clear pattern. First, near-OOD ImageNet-scale detection remains difficult for all methods, especially on SSB-hard. Second, Textures is substantially more discriminative and reveals a clear advantage for sparse \cfs\ probes over \msma\ and \diffpath. Third, under this official improved-backbone setting, the sparse decoder-only variant \(\cfs_{\mathrm{dec}}(1{\times}1)\) is consistently slightly stronger than \(\cfs(1{\times}2)\) on both ImageNet200 and ImageNet1K, while preserving the same \(1F\) logical cost.

Interestingly, in the ImageNet-scale checkpoint-controlled setting, \(\cfs_{\mathrm{dec}}(1{\times}1)\) is slightly stronger than \(\cfs(1{\times}2)\). 
The result suggests that, in this harder large-scale regime, the conditional encoder residual of~Theorem~\ref{thm:conditional_complementarity_main}, is either small or not reliably exploitable by the lightweight diagonal score, so the late decoder snapshot remains the most robust sparse probe.

The NINCO and SSB-hard results indicate that large-scale semantic OOD remains difficult for all diffusion-based methods under this fixed single-backbone evaluation setting.

\section{Implementation and Configuration Details}
\label{app:impl}

This section reports the implementation details needed to reproduce the main \cfs\ runs and the corresponding ablations. All methods are evaluated through the shared \mbe\ adapter described in Appendix~\ref{app:protocols}. Unless stated otherwise, images are normalized to \([-1,1]\), the test-time Monte Carlo count is one, and all reported main-paper methods use OOD-high scores.

\subsection{CFS probing configuration}
\label{app:cfs_config}

Table~\ref{tab:cfs_probe_config} reports the sparse probing configuration used for the primary one-level \cfs{} runs. The same canonical probing setup is used for \(\cfs(1{\times}2)\), while decoder-only variants restrict the retained region to the decoder side.

\begin{table}[t]
\centering
\caption{
\textbf{Sparse \cfs\ probing configuration.}
This configuration specifies the canonical level, hook policy, pooling rule, Monte Carlo settings, and ID-only hook-selection proxy used for the main sparse probing runs.
}
\label{tab:cfs_probe_config}
\small
\setlength{\tabcolsep}{5pt}
\begin{tabular}{ll}
\toprule
Parameter & Value \\
\midrule
Number of selected canonical levels \(K_c\) & \(1\) \\
Candidate canonical grid \(K_{\mathrm{grid}}\) & \(1\) \\
Explicit canonical levels & \([5.0]\) \\
Pooling rule & channel-wise mean + standard deviation \\
Dynamics features & disabled \\
Hook policy & \texttt{sparse\_ed\_id} \\
Region mode for \(\cfs(1{\times}2)\) & encoder + decoder \\
Retained encoder hooks & \(1\) \\
Retained decoder hooks & \(1\) \\
Fit-time Monte Carlo count \(\mathrm{MC}_{\mathrm{fit}}\) & \(1\) \\
Test-time Monte Carlo count \(\mathrm{MC}_{\mathrm{test}}\) & \(1\) \\
Maximum fit batches & \(64\) \\
Internal batch size & \(64\) \\
ID probe batches for hook proxy & \(1\) \\
Maximum candidates per region & \(4\) \\
Clamping & disabled \\
Numerical floor \(\varepsilon\) & \(10^{-6}\) \\
\bottomrule
\end{tabular}
\end{table}

For the main results, each selected slot is scored with the lightweight diagonal statistic in  Eq.~\eqref{eq:slot_score_new}. Table~\ref{tab:cfs_main_head_config} reports the corresponding ID-only head configuration. For the head-sensitivity ablation in Appendix~\ref{app:ablation_heads}, we additionally evaluate stronger ID-only heads on the same sparse representation.

\begin{table}[t]
\centering
\caption{
\textbf{Default ID-only head configuration for the main \cfs{} results.}
The main paper uses the diagonal score to keep the detector lightweight and to avoid conflating representation quality with downstream head capacity.
}
\label{tab:cfs_main_head_config}
\small
\setlength{\tabcolsep}{5pt}
\begin{tabular}{ll}
\toprule
Parameter & Value \\
\midrule
Head type & diagonal Gaussian / diagonal Mahalanobis score \\
Slot statistics & ID-only mean and diagonal variance \\
Variance floor & \(10^{-6}\) \\
Feature pooling & channel-wise mean + standard deviation \\
Slot aggregation & uniform average over selected slots \\
KNN head & not used in main results \\
GMM head & not used in main results \\
Shrinkage covariance head & not used in main results \\
\bottomrule
\end{tabular}
\end{table}

\paragraph{Pooling.}
The main-paper pooling rule is the one defined in Eq.~\eqref{eq:pooled_feature_new}: channel-wise spatial mean concatenated with channel-wise spatial standard deviation. Mean-only and std-only variants are treated as pooling ablations in Appendix~\ref{app:ablation_pooling}, not as the default \cfs{} configuration.

\subsection{Compute profiling protocol}
\label{app:compute_profile}

We profile wall-clock inference cost using the same benchmark runner as the reported experiments in Table~\ref{tab:compute_profile}. For each method, profiling is performed after ID-only fitting and measures the full scoring path: canonical corruption, backbone evaluations, hook or output extraction, and score computation. We use \(5\) warm-up batches and \(50\) measured batches, with CUDA synchronization before and after the measured region. Peak memory is measured with \texttt{torch.cuda.max\_memory\_allocated}. All timings are reported at the same batch size used in the benchmark.

We distinguish logical cost from measured runtime. Logical cost counts backbone evaluations per image, \(\#F\); measured runtime additionally includes feature pooling, density-head evaluation, data movement, and method-specific aggregation overhead. 

\begin{table}[t]
\centering
\caption{
\textbf{Representative compute profile on the CIFAR-scale benchmark.}
Timings are measured after ID-only fitting using the full scoring path with \(5\) warm-up batches and \(50\) measured batches. GPU-hours estimate the fit plus evaluation cost for one ID-vs-OOD benchmark run under the reported split sizes.
}
\label{tab:compute_profile}
\small
\setlength{\tabcolsep}{4.5pt}
\resizebox{\textwidth}{!}{%
\begin{tabular}{llccccc}
\toprule
Method & Backbone & Logical cost & Fit time (s) & ms/img & Peak mem. (GB) & GPU-hours \\
\midrule
\msma & improved & 10F & 160.8 & 19.52 & 0.66 & 0.41 \\
\diffpath & improved & 10F & 163.1 & 19.70 & 0.66 & 0.42 \\
\ddpmood & improved & 364F & 5851.2 & 724.83 & 0.66 & 15.29 \\
\gepc & improved & 8F & 129.0 & 16.26 & 0.67 & 0.34 \\
\(\cfs_{\mathrm{dec}}(1{\times}1)\) & improved & 1F & 26.5 & 3.23 & 0.78 & 0.07 \\
\(\cfs(1{\times}2)\) & improved & 1F & 26.9 & 3.20 & 0.79 & 0.07 \\
\midrule
\msma & EDM & 10F & 506.3 & 64.12 & 0.84 & 1.35 \\
\diffpath & EDM & 10F & 473.9 & 57.89 & 0.84 & 1.22 \\
\ddpmood & EDM & 364F & -- & -- & -- & -- \\
\gepc & EDM & 8F & 381.7 & 48.23 & 0.85 & 1.02 \\
\(\cfs_{\mathrm{dec}}(1{\times}1)\) & EDM & 1F & 53.3 & 6.42 & 0.85 & 0.14 \\
\(\cfs(1{\times}2)\) & EDM & 1F & 49.5 & 5.88 & 0.87 & 0.12 \\
\bottomrule
\end{tabular}}
\end{table}

All profiles in Table~\ref{tab:compute_profile} were obtained on a single NVIDIA GeForce RTX 4060 Laptop GPU with batch size \(128\), and PyTorch 2.10.0. For \ddpmood{} with EDM, we report the logical cost but omit wall-clock profiling because the \(364F\) scoring path is prohibitively expensive; the improved-diffusion profile already illustrates this order-of-magnitude cost.

\section{Broader Impact and Responsible Use}
\label{app:broader_impact}

This work studies post-hoc OOD detection for frozen diffusion backbones. A potential positive impact is improved reliability and auditability of vision systems: sparse internal probes may help identify inputs that fall outside an evaluation reference distribution, while keeping test-time cost low and making protocol confounders explicit.

The main negative risk is over-reliance. An OOD score is not a certificate of safety, correctness, or fairness, and false negatives may create unwarranted confidence in high-stakes settings. This is especially important for sensitive visual domains such as medical imaging, biometric analysis, surveillance, or safety-critical autonomy, where dataset shift, demographic imbalance, or acquisition bias may produce failures not captured by the evaluation protocol.

Our method is intended as a diagnostic and benchmarking tool rather than a standalone deployment safeguard. In practical use, it should be combined with domain-specific validation, calibrated thresholds, uncertainty auditing, and human oversight where decisions may affect people. 

\end{document}